\documentclass[journal]{IEEEtran}

\IEEEoverridecommandlockouts                              
\usepackage{pifont}
\usepackage{stfloats}
\usepackage{breqn}
\usepackage{cuted}
\usepackage{multirow}
\usepackage{booktabs}
\usepackage{threeparttable}
\usepackage{subcaption}
\usepackage{algorithm}
\usepackage{algpseudocode}
\usepackage{docmute}
\usepackage{lineno,hyperref}
\usepackage{amsmath,graphicx}
\usepackage{amsthm}

\usepackage{amssymb}
\usepackage{booktabs}
\usepackage{cite}
\usepackage{tablefootnote}
\usepackage{xcolor}
\usepackage{mathtools}
\usepackage{cases}
\usepackage{tabularx}
\usepackage{caption}  
\usepackage[final]{changes}
\usepackage{amsthm}
\newtheorem{theorem}{Theorem}
\newtheorem{lemma}{Lemma}

\newtheorem{assumption}{Assumption}
\newtheorem{remark}{Remark}

\newtheorem{prop}{Proposition}

\newcommand\ox{\overline{x}}

\newcommand\bx{\mathbf{x}}

\newcommand\bE{\mathbb{E}}

\newcommand\obx{{\mathbf{x}}}

\newcommand\by{\mathbf{y}}

\newcommand\bz{{\mathbf{z}}}

\newcommand\nbf{\nabla \mathbf{f}}

\definecolor{purple}{rgb}{0.6, 0.2, 0.8}

\usepackage{multibib}
\definecolor{UniBlue}{RGB}{83,121,170}
\definecolor{DarkGray}{RGB}{90,90,90}
\definecolor{LightGray}{RGB}{150,150,150}
\definecolor{oldTextGreen}{RGB}{115,155,15}
\definecolor{teal}{RGB}{100, 200,10}
\definecolor{oldOcean}{RGB}{23,142,189}
\definecolor{Ocean}{RGB}{30,106,181}
\definecolor{BG}{RGB}{215,215,215}
\definecolor{darkred}{RGB}{204,41,0}

\usepackage{hyperref}

\allowdisplaybreaks

\title{An Efficient Subspace Algorithm for Federated Learning on Heterogeneous Data}

\author{Jiaojiao Zhang, Yizhao Fan,  Yuqi Xu and Kun Yuan$^*$
\thanks{Corresponding author is Kun Yuan. Kun Yuan was supported by the National Natural Science Foundation of China (No. 92370121) and the National Key Research and Development Program of China (No. 2024YFA1012902). Jiaojiao Zhang was supported by the National Natural Science Foundation of China (No. 12601589). Jiaojiao Zhang is with the Intelligent Computing Research Center, Great Bay University, Dongguan, China. Yizhao Fan is with the School of Computer Science and Technology, Harbin Institute of Technology, Shenzhen, and the Intelligent Computing Research Center, Great Bay University, Dongguan, China. Yuqi Xu and Kun Yuan are with the Center for Machine Learning Research, Peking University, Beijing, China. }
\thanks{This paper has supplementary downloadable material available at http://ieeexplore.ieee.org, provided by the authors. The material includes additional proofs and experimental details.}
}

\begin{document}

\maketitle

\begin{abstract}
This work addresses the key challenges of applying federated learning to large-scale deep neural networks, particularly the issue of client drift due to data heterogeneity across clients and the high costs of communication, computation, and memory.
We propose {\bf FedSub},   an efficient subspace algorithm for federated learning on heterogeneous data. Specifically, FedSub utilizes subspace projection to guarantee local updates of each client within low-dimensional subspaces, thereby reducing communication, computation, and memory costs. Additionally,
it incorporates low-dimensional dual variables to mitigate
client drift.  Theoretically, we provide convergence analysis that reveals the impact of key factors such as step size, subspace projection matrices, and stochastic gradient noise on convergence. Comprehensive experiments demonstrate its efficiency.
\end{abstract}

\section{Introduction}

Federated learning (FL) enables collaborative model training across distributed clients without sharing raw data \cite{fedavg,zhang2024hybrid,wang2023batch}. It has been applied to large-scale model training, including training large language models (LLMs) \cite{chen2025memory}, by leveraging distributed data and parallel computation.

Despite these advantages, FL faces several challenges in large-scale model training. One key issue is data heterogeneity, where non-IID client data lead to client drift and degrade the performance of the global model \cite{karimireddy2020scaffold}. Furthermore, edge clients typically suffer from limited communication bandwidth, computation power, and memory capacity. 
These limitations are exacerbated by the high dimensionality of large-scale models, resulting in severe communication, computation, and memory bottlenecks that hinder the scalability of FL in practical deployments \cite{LSJ-2025Rethinking,sub-mahla2024exploring}.

In single-machine settings, subspace methods have demonstrated effectiveness in improving training efficiency by projecting model updates onto carefully constructed low-dimensional subspaces \cite{hu2022lora,fedlora-1-lee2025fedsvd,zhao2024galore,he2024subspace}. This approach leverages the empirical observation that neural network parameters and gradients often exhibit low-rank structures during training \cite{he2024subspace}. 
However, single-machine subspace algorithms are not directly applicable to FL, as they lack mechanisms to address client drift under data heterogeneity and do not support efficient communication in distributed environments. 

Some recent studies integrate subspace methods with FL \cite{sub-mahla2024exploring,sub-bai2024federated,sub-sun2024improving,fedlora-1-lee2025fedsvd,zhihuazhang-subspace}, but they do not fully address client drift or simultaneously improve communication, computation, and memory efficiency. These limitations highlight the need for federated subspace algorithms specifically designed for heterogeneous data, which motivates the approach presented in this work.

\subsection{Contributions}
We propose an efficient {\bf Sub}space algorithm for {\bf Fed}erated learning on heterogeneous data, abbreviated as {\bf FedSub}. The main idea is to use low-dimensional subspace projections to improve communication, computation, and memory efficiency, while introducing dual variables that lie in low-dimensional subspaces to mitigate the client drift. The key contributions are summarized as follows:

(i) We use subspace projection to ensure that local updates of each client remain within low-dimensional subspaces, leading to simultaneous reductions in communication, computation, and memory costs. Additionally, we incorporate dual variables that also lie in low-dimensional subspaces to mitigate client drift. The efficiency of our FedSub is summarized in Table \ref{tab}. When the projection matrix becomes the identity, i.e., $P^k = \mathrm{I}$, FedSub naturally reduces to a full-space stochastic primal–dual FL algorithm. 

(ii) We provide convergence analysis that characterizes the impact of key factors such as the step size, subspace projection matrices, and stochastic gradient noise on convergence. When setting the projection matrix as the identity, our theoretical results recover the convergence properties of the full-space stochastic primal-dual FL algorithm.

(iii) We conduct comprehensive experiments to validate the effectiveness of our FedSub. 

\begin{table*}[]
\centering
\caption{Assume a single layer. $\mathcal{O}$-notation is omitted for simplicity. $C_{g}(rd)$ denotes the computation cost of evaluating a gradient with size $r \times d$. $M_g(rd)$ denotes the memory cost of computing a gradient of dimension $r\times d$ including intermediate quantities such as activations. As shown in Remark \ref{remark-scaffold}, although FedSub-$P^k = \mathrm{I}$ is mathematically equivalent to the well-known Scaffold \cite{karimireddy2020scaffold}, Scaffold requires each client to send correction terms to the server, which increases communication overhead.} 
\label{tab}
\begin{tabular}{c c c c c}
\toprule
              & FedSub (Algorithm \ref{alg-fl}) & FedSub-$P^k = \mathrm{I}$   & FedAvg  & Scaffold  \\ \hline
communication (uplink) &  $rd$   &  $md$   & $md$   & $2 md$   \\ \hline
computation   & $ \tau( mrd+C_{g}(rd) )+ 2mrd $    &  $\tau C_{g}(md)$   &  $\tau C_{g}(md)$ & $\tau C_{g}(md)$ \\ \hline
memory        &  $3rd+M_g(rd)+2rm+md$   & $3md+M_g(md)$     &    $md+M_g(md)$    & $3md+M_g(md)$  \\ 
\bottomrule
\end{tabular}
\end{table*}

\subsection{Related Work}
We organize the related work into three categories: full-space federated learning, subspace single-machine learning, and subspace federated learning.

\subsubsection{Full-Space Federated Learning} 
FedAvg \cite{fedavg} is one of the most widely used FL algorithms. It improves communication efficiency by allowing clients to perform multiple local updates before sending messages to the server. However, in real-world scenarios where data is often heterogeneous, performing multiple local updates on such heterogeneous data leads to client drift, which degrades training accuracy.
To address the challenges induced by heterogeneous data, various strategies have been proposed in FL, including approaches based on personalization\cite{tan2022towards}, regularization\cite{li2020federated,an2023federated}, momentum\cite{cheng2024momentum}, correction mechanisms\cite{karimireddy2020scaffold,karimireddy2020mime,FedLin}, and primal-dual formulations\cite{fedPD,rev2-feddyn}; a comprehensive review of approaches can be found in \cite{karami2025harmony-review}.
A representative correction-based method is Scaffold \cite{karimireddy2020scaffold}, which introduces control variates to reduce the discrepancy between local and global updates and effectively mitigates client drift. 
Similarly, Mime \cite{karimireddy2020mime} and FedLin \cite{FedLin} incorporate correction mechanisms to improve convergence under heterogeneous data. Besides correction-based approaches, FedPD \cite{fedPD} adopts a primal-dual framework to handle the inconsistency between local and global objectives. 
Our method is derived from a primal-dual formulation. Moreover, under the full-space setting, we prove that our FedSub admits a correction-based interpretation and is algebraically equivalent to Scaffold, as discussed in Remark~\ref{remark-scaffold}.

However, the efficiency of these full-space FL methods remains limited when training large-scale models in terms of communication, computation, and memory usage. For example, the high dimensionality of model parameters leads to substantial communication overhead during model transmission, as well as significant computation and memory costs during backpropagation for gradient computation. Although compression, such as quantization   \cite{tang2024z, lan2023quantization} and sparsification \cite{khirirat2020compressed}, can alleviate communication costs, they typically apply compression after computing the full-space parameters. Consequently, they do not reduce the computation cost and memory cost.

\subsubsection{Subspace Single-machine Learning}

Training large-scale models, such as LLMs, requires significant memory and computation, primarily due to model parameters, optimizer states, and activations. To address these challenges, recent work has proposed strategies that reduce the number of trainable model parameters, shrink the size of optimizer states, or lower the memory needed to store activations.

{\bf (i)} Model Parameters. One major research direction focuses on reducing the number of trainable model parameters. LoRA \cite{hu2022lora} achieves this by decomposing parameters into a frozen pre-trained component and a low-rank trainable part, updating only the latter. However, the low-rank constraint restricts LoRA to fine-tuning and limits its applicability to pre-training tasks. To overcome this, several improved LoRA variants have been proposed \cite{lialin2023relora,xia2024chain-lora, pan2024lisa}. For instance, ReLoRA \cite{lialin2023relora} approximates full-rank updates by accumulating multiple low-rank updates. Similarly, sparsity-based methods \cite{thangarasa2023spdf} train only a subset of model parameters. These approaches effectively lower the number of trainable parameters, reducing memory usage for gradients and optimizer states.

{\bf (ii)} Optimizer States. 
Another line of work focuses on reducing the memory of optimizer states. GaLore and its variants \cite{zhao2024galore,hao2024flora,chen2024enhancing-0order,liang2024memory} 
 compress gradients by projecting them into low-rank subspaces via singular value decomposition (SVD) of the true full-space gradients and use these compressed gradients to compute first- and second-moments, reducing their memory usage. However, GaLore may fail to converge under standard assumptions. To address this, GoLore \cite{he2024subspace} modifies the SVD-based projection in GaLore to a random subspace projection and proves convergence under standard assumptions. Furthermore, Adam-mini \cite{zhangadam-mini} reduces memory by using block-wise second-order moments for adaptive learning rates. While these methods reduce memory for optimizer states, they do not reduce memory for activation since they still require storing full activations to compute full-space gradients.

{\bf(iii)} Activations. To reduce memory cost from activations during gradient computation, zero-order (ZO) algorithms approximate gradients via finite differences of function values, eliminating the need for backpropagation and storage of activations \cite{gautam2024variance-0order,chen2024enhancing-0order,zhao2024second,LSJ-zhangrevisiting}. While this  lowers memory usage, ZO methods generally converge slowly and are mainly applied to fine-tuning. Alternatively, activation compression techniques \cite{jiang2022back-razor,yu2024sheared} reduce memory demands by representing activations as sparse vectors during training.
The work \cite{chen2025memory} proposes a method called Random Subspace Optimization (RSO), which decomposes LLMs training into a sequence of low-dimensional subproblems,  thereby reducing memory usage from both optimizer states and activations.

While these centralized subspace learning reduce memory and computation costs, they cannot be directly extended to FL under heterogeneous data, where communication bottleneck and the challenge of client drift must also be addressed.

\subsubsection{Federated Subspace Learning}

Although many studies have extended LoRA to FL \cite{sub-sun2024improving, sub-bai2024federated, sub-park2024communication, fedlora-1-lee2025fedsvd, fedlora-2-liu2025differentially, fedlora-3-grativol2024flocora}, these approaches are better suited for fine-tuning and less effective for pre-training tasks. Moreover, they fail to reduce communication, computation, and memory costs simultaneously.
For example, FlexLoRA \cite{sub-bai2024federated} supports heterogeneous LoRA ranks across clients, but requires each client to transmit two low-rank matrices per communication round. Additionally, as pointed out by \cite{sub-mahla2024exploring}, under heterogeneous data, the aggregated global model may not retain a low-rank structure \cite{sub-mahla2024exploring}, making enforced low-rank projections potentially harmful to performance.
FedLoRU \cite{sub-park2024communication} uses low-rank matrix factorization, but clients must solve constrained subspace optimization problems exactly and transmit two factor matrices, leading to high computation and communication costs. Moreover, it does not effectively address client drift.

Beyond LoRA-based methods, FedFTG \cite{sub-mahla2024exploring} extends GaLore to the federated setting by integrating it into the FedAvg framework. However, each client still transmits the full-space models to the server,  failing to reduce communication costs. Additionally, FedFTG cannot overcome client drift and provides convergence guarantees only for convex objectives.
%

{\bf Notation.} 
Define $[n] = \{1, \ldots, n\}$. For a neural network with $L$ layers,  $x_l \in \mathbb{R}^{m_l \times d_l}$ is the weight matrix of the $l$-th layer, and the collection of all layer weights is $x = \{x_l\}_{l=1}^L$. Given $n$ such collections $x_1, \ldots, x_n$ where $x_i = \{x_{i,l}\}_{l=1}^L$, their layerwise average is $\frac{1}{n}\sum_{i=1}^n x_i = \left\{\frac{1}{n}\sum_{i=1}^n x_{i,l}\right\}_{l=1}^L$, with addition and subtraction defined similarly layerwise. For two sets of scalars $m = \{m_l\}_{l=1}^L$ and $r = \{r_l\}_{l=1}^L$, $m \times r$ denotes the layerwise dimensions, with the $l$-th layer of size $m_l \times r_l$. Layerwise scalar multiplication is defined as $\frac{m}{r} x := \{\frac{m_l}{r_l} x_l\}_{l=1}^L$, and we define $\theta_m = \max_l \frac{m_l}{r_l}$, $\theta_r = \max_l \frac{r_l}{m_l}$. For a function $f(x)$, the gradient is taken layerwise as $\nabla f(x) = \{\nabla_{x_l} f(x)\}_{l=1}^L$ where $\nabla_{x_l} f(x)$ denotes the partial derivative of $f$ with respect to $x_l$. Given $P = \{P_l\}_{l=1}^L$ with $P_l \in \mathbb{R}^{m_l \times r_l}$, the layerwise transpose and multiplication with the gradient are $P^T = \{P_l^T\}_{l=1}^L$ and $P^T \nabla f(x) := \{P_l^T \nabla_{x_l} f(x)\}_{l=1}^L$. The norm $\|x\|^2$ is defined as $\sum_{l=1}^L \|x_l\|_F^2$, where $\|\cdot\|_F$ is the Frobenius norm, and $\mathbb{E}[\cdot]$ denotes expectation. The identity matrix $I$ refers to a matrix or collection of identity matrices of appropriate size depending on context. For client $i$,  \(\hat{\nabla} f_i\) denotes the stochastic gradient, with \(\nabla f_i\) being the corresponding full gradient.

\section{Preliminary}

\subsection{Problem Formulation in Deep Learning}
We consider training a Deep Neural Network (DNN) with $L$ layers in FL, coordinated by a central server and distributed across $n$ clients. Each client $i \in [n]$ holds a local dataset $D_i$ and defines a local loss function $f_i(x)$.
The global objective is to minimize the average loss across all clients:
\begin{equation}\label{eqn:basic_opt}
\min_{x = \{x_l\}_{l=1}^L} \; f(x) = \frac{1}{n} \sum_{i=1}^n f_i(x),
\end{equation}
where  $x = \{x_l\}_{l=1}^L$ denotes the collection of NN weights across all layers.

To solve \eqref{eqn:basic_opt} in a distributed manner, we introduce a local copy of the weights for each client, denoted as $x_i = \{x_{i,l}\}_{l=1}^L$, where $x_{i,l} \in \mathbb{R}^{m_l \times d_l}$ is the local weight matrix for layer $l$ on client $i$. We reformulate \eqref{eqn:basic_opt} as:
\begin{equation}\label{eqn:opt_consensus}
\begin{aligned}	
\min_{\{x_1,\ldots,  x_{n}\}}~ \frac{1}{n}\sum_{i=1}^n f_i(x_i),\quad\text{s.t.}~  x_i-\frac{1}{n} \sum_{i=1}^{n} x_i=0, ~\forall i.
\end{aligned}
\end{equation}
The consensus constraint of \eqref{eqn:opt_consensus} ensures that all local weights are equal to the global average, i.e., $x_1 = \cdots = x_n = \frac{1}{n} \sum_{i=1}^n x_i$.
For a more compact representation, we define  $\bx = \{x_i\}_{i=1}^n$ to aggregate all local weights and introduce a linear operator $\mathcal{L}$ to encode the consensus constraint as:
\begin{equation}\label{eq-def-L}
(\mathcal{L} \bx)_i := x_i - \frac{1}{n} \sum_{i=1}^n x_i, \quad \forall i \in [n].
\end{equation}
Using \eqref{eq-def-L} and expressing the objective in \eqref{eqn:opt_consensus} as $f(\bx)$, we rewrite \eqref{eqn:opt_consensus} in a compact form:
\begin{equation}\label{eqn:opt_consensus_compact}
\begin{aligned}
\min_{\bx}~ f(\bx), \quad \mbox{s.t.}~ \mathcal{L} \bx=0.
\end{aligned}
\end{equation}

As we will show in Section \ref{sec-derive}, by handling the consensus constraint in \eqref{eqn:opt_consensus_compact} via dual variables (specifically, variables $\{\Lambda_i^k\}_k$ in Algorithm \ref{alg-fl}), we can derive an algorithm that uses these dual variables to mitigate the client drift.

\subsection{Subspace Projection}
To improve the efficiency of FL, we restrict local updates on each client to low-dimensional subspaces. For each layer $l \in [L]$, let $P_l \in \mathbb{R}^{m_l \times r_l}$ denote a subspace projection matrix, where $r_l \ll m_l$. A full-space variable $z_l \in \mathbb{R}^{m_l \times d_l}$ is projected onto a low-dimensional subspace via $P_l^T z_l \in \mathbb{R}^{r_l \times d_l}$. 
Inspired by \cite{he2024subspace,chen2025memory}, we employ random subspace projection matrices and make the following assumption:
\begin{assumption}\label{asm-p}
For each layer $l$ and communication round $k$, the random subspace projection matrices $P_l^k$ are independent for all $l,k$ and satisfy:
\[
(P_l^k)^T P_l^k = \frac{m_l}{r_l} I, \quad \text{and} \quad \mathbb{E}[P_l^k (P_l^k)^T] = I, \quad \forall l, k,
\]
where the expectation $\mathbb{E}$ is taken over the randomness of the projection matrix $P_l^k$.
\end{assumption}

Random projection matrices satisfying Assumption~\ref{asm-p} can be constructed as described in~\cite{he2024subspace}; specific examples are provided in Section \ref{sec-experiment-lr}.

\section{Algorithm}
We propose an efficient subspace FL algorithm that leverages low-dimensional subspace projections to reduce communication, computation, and memory overhead, while use dual variables to mitigate client drift caused by heterogeneous data. 

\begin{algorithm}[htbp]
\caption{Proposed Algorithm (FedSub)}
\label{alg-fl}
\begin{algorithmic}[1]
\State $ \textbf{Input:}$ $K$, $\tau$, $\eta$, $P^k=\{P_l^k\}_{l=1}^L$, $\Lambda_i^0=0$, $ x^{0}$
\For {$k = 0, 1, \ldots, K-1$}
    \State \textbf{Client $i$} 
    \State Set $B_i^{k,0} = 0$
    \For {$t = 0, 1, \ldots, \tau-1$}
        \State Sample a subset of size $b$ from the dataset $D_i$
        \State  Compute low-dimensional stochastic gradients $(P^k)^T\hat{\nabla} f_i(x^k + P^k B_i^{k,t})$ 
        \State Update the variable in the subspace
        $$
        B_i^{k,t+1} = B_i^{k,t} - \eta \left( \frac{r}{m} (P^k)^T \hat{\nabla} f_i(x^k + P^k B_i^{k,t}) + \frac{\Lambda_i^k}{\eta \tau}  \right)
        $$
    \EndFor
    \State Send low-dimensional $B_i^{k,\tau}$ to the server
    \State \textbf{Server} 
    \State Compute $\frac{1}{n} \sum_{i=1}^n B_i^{k,\tau}$ and broadcast it to all clients
    \State \textbf{Client $i$}
    \State Update the dual variable in the subspace
    $$\Lambda_i^{k+1} = (P^{k+1})^T \left(P^k \left( \Lambda_i^k + B_i^{k,\tau} - \frac{1}{n} \sum_{i=1}^n B_i^{k,\tau} \right) \right)
    $$
    \State \textbf{Server} 
    \State Update 
    $
    x^{k+1} = x^k + P^k \frac{1}{n} \sum_{i=1}^n B_i^{k,\tau}
    $
    \State Broadcast $x^{k+1}$ to all clients
\EndFor
\State \textbf{Output} $x^K$
\end{algorithmic}
\end{algorithm}

\subsection{Proposed Algorithm}\label{sec-derive}

Our FedSub is presented in Algorithm~\ref{alg-fl}.  
FedSub proceeds over \( K \) communication rounds. In each communication round \( k \), client \( i \) performs \( \tau \) local updates within a low-dimensional subspace defined by the projection matrix \( (P_l^k)^T \in \mathbb{R}^{r_l \times m_l} \) with $r_l \ll m_l$.  
The local subspace variable \( B_i^{k,t} \) is updated by 
\[
B_i^{k,t+1} = B_i^{k,t} - \eta \left( \frac{r}{m} (P^k)^T \hat{\nabla} f_i(x^k + P^k B_i^{k,t}) + \frac{1}{\eta \tau} \Lambda_i^k \right),
\]
where $P^k = \{P_l^k\}_{l=1}^L$ and the operations is layer-wise as defined in Notation. 
For all \( t \in [\tau]-1 \), both  \( B_i^{k,t} \) and \( \Lambda_i^k \) lie in the subspace defined by \( (P_l^k)^T \) for each layer $l$. 

After local updates, each client sends \( B_i^{k,\tau} \) to the server, which computes the average $ \frac{1}{n} \sum_{i=1}^n B_i^{k,\tau}$ and broadcasts it to all clients. Each client then updates the dual variable via: 
\[ 
\Lambda_i^{k+1} = (P^{k+1})^T \left(P^k \left( \Lambda_i^k + B_i^{k,\tau} - \frac{1}{n} \sum_{i=1}^n B_i^{k,\tau} \right)\right), 
\] 
where \( (P^{k+1})^T P^k \) ensures that \( \Lambda_i^{k+1} \) lies in the new subspace of \( (P^{k+1})^T \). 

The server updates the global model in full space  via
\[x^{k+1} = x^k + P^k \frac{1}{n} \sum_{i=1}^n B_i^{k,\tau}\]
and broadcasts it to all clients.  

\subsection{Derivation of Algorithm \ref{alg-fl}}
FedSub is derived from the primal-dual hybrid gradient (PDHG) method \cite{alghunaim2022unified}. In the following, we present the derivation of FedSub via a two-stage procedure: first applying the primal–dual method to problem~\eqref{eqn:opt_consensus_compact} to obtain a full-space algorithm, and then converting it into a subspace variant. 

The primal-dual formulation of problem~\eqref{eqn:opt_consensus_compact} is
\begin{equation}\label{eq-saddle-point}
\min_{\bx} \max_{\by } f(\bx) + \langle \by, \mathcal{L} \bx \rangle,
\end{equation}
where $\by = \{y_i\}_{i=1}^n$ with each $y_i = \{y_{i,l}\}_{l=1}^L$ representing the dual variable, and 
\[
\langle \by, \mathcal{L}\bx \rangle = \sum_{i=1}^n y_i^T (\mathcal{L}\bx)_i.
\]

To solve \eqref{eq-saddle-point}, we employ an alternating optimization strategy. First, the primal variable $\bx$ is updated approximately while keeping the dual variable $\by$ fixed. Then, the dual variable is updated via a gradient ascent step with $\bx$ fixed. 

To further enhance communication efficiency, each client performs $\tau$ local updates before communicating with the server. Incorporating this strategy into the primal-dual framework yields the following algorithm:
\begin{equation}\label{eq-pd-local}
\begin{aligned}
&\bz^{k,t+1}=\bz^{k,t}-\eta \nbf(\bz^{k,t})-\frac{1}{\tau}{\mathcal{L} \by^k}, \\		
&\by^{k+1}=\by^k+\mathcal{L}  \bz^{k,\tau},\\
&\obx^{k+1}=\bz^{k,\tau}-\mathcal{L}\bz^{k,\tau}.
\end{aligned}
\end{equation}
Here, $\bz = \{z_i\}_{i=1}^n$ with $z_i = \{z_{i,l}\}_{l=1}^L$, $\nbf(\bz^{k,t}) = \{\nabla f_i(z_i^{k,t})\}_{i=1}^n$, $\bz^{k,0} = \bx^k$, and $\by^0 = 0$.  
Due to the structure of $\mathcal{L}$, all blocks of $\bx^k$ are identical, so we denote a single block of $\bx^k$  by $x^k$.  
When $\tau = 1$, \eqref{eq-pd-local} reduces to PDHG \cite{alghunaim2022unified}.

Next, we adapt \eqref{eq-pd-local} to a subspace setting.  
Let $P^k B_i^{k,t} = z_i^{k,t} - x^k$; then, \eqref{eq-pd-local} can be expressed as
\begin{align}\label{eq-zz}
&P^kB_i^{k,t+1}=P^kB_i^{k,t}-\eta \left(\nabla f_i(x^k+ P^kB_i^{k,t})+\frac{y_i^k}{\eta\tau}  \right),\notag\\ 
&y_i^{k+1}=y_i^k+P^k\left(B_i^{k,\tau}-\frac{1}{n} \sum_{i=1}^n B_i^{k,\tau}\right),\\\notag
&x^{k+1}=x^k+P^k\frac{1}{n}\sum_{i=1}^n B_i^{k,\tau},
\end{align}    
where we substitute $z_i^{k,t}=x^k+P^kB_i^{k,t}$ in the three equalities in \eqref{eq-zz}. Besides, we use $\mathcal{L} \by^k =\by^k$  in the first equality and use  $\mathcal{L}\bx^k=0$ in the second and the last equalities.

Since $(P^k)^T P^k = \frac{m}{r} I$, left-multiplying both sides of the above update by $(P^k)^T$ gives
\begin{align}\label{eq-pd-11}
&\frac{m}{r} B_i^{k,t+1}=\frac{m}{r} B_i^{k,t}-\eta (P^k)^T\left(  \nabla f_i(x^k+ P^kB_i^{k,t})+\frac{y_i^k}{\eta\tau}  \right),\notag\\ 
&(P^k)^T y_i^{k+1}=(P^k)^T y_i^k+\frac{m}{r}\left(B_i^{k,\tau}-\frac{1}{n} \sum_{i=1}^n B_i^{k,\tau}\right),\\ \notag
&x^{k+1}=x^k+P^k\frac{1}{n}\sum_{i=1}^n B_i^{k,\tau}.
\end{align}    
For the second update in \eqref{eq-pd-11}, we have
\begin{align}\label{eq-p-8}
&(P^{k+1})^T P^k (P^k)^T y_i^{k+1}\\\notag
&=(P^{k+1})^T P^k\left( (P^k)^T y_i^k+\frac{m}{r}\left(B_i^{k,\tau}-\frac{1}{n} \sum_{i=1}^n B_i^{k,\tau}\right) \right).    
\end{align}   
Since $\mathbb{E}[P^k (P^k)^T] = I$, in FedSub we simplify \eqref{eq-p-8} by dropping the term $P^k (P^k)^T$ on the left-hand side and defining 
\[
\Lambda_i^k = \frac{r}{m} (P^k)^T y_i^k.
\] 
Then, using \eqref{eq-pd-11} and \eqref{eq-p-8}, we have
\begin{align}\label{our-derive}
&B_i^{k,t+1}= B_i^{k,t}- \eta \left( \frac{r}{m} (P^k)^T \nabla f_i(x^k+ P^kB_i^{k,t})+\frac{\Lambda_i^k}{\eta\tau}  \right), \notag\\
&\Lambda_i^{k+1}= (P^{k+1})^T P^k \left(\Lambda_i^k+ B_i^{k,\tau}-\frac{1}{n} \sum_{i=1}^n B_i^{k,\tau}\right),\\\notag
&x^{k+1}=x^k+P^k\frac{1}{n}\sum_{i=1}^n B_i^{k,\tau}.    
\end{align}    
By replacing the full gradients $\nabla f_i$ with stochastic estimates $\hat{\nabla} f_i$, we obtain the proposed Algorithm~\ref{alg-fl}.

\subsection{Dual Variables to Mitigate Client Drift}
Using dual variables $\{\Lambda_i^k\}_k$ enables FedSub to mitigate client drift. 
In Appendix \ref{sec-simplify}, we prove that  
\begin{equation}
\Lambda_i^{k}= (P^{k})^T P^{k-1}\eta \sum_{t=0}^{\tau-1} \left( \bar{g}(B_i^{k-1,t}) - g_i(B_i^{k-1,t})\right),    
\end{equation}
where ${g_i(B_i^{k-1,t})}={\frac{r}{m}} (P^{k-1})^T\hat{\nabla} f_i(x^{k-1}+P^{k-1} B_i^{k-1,t})$ and $\bar{g}(B_i^{k-1,t})=\frac{1}{n}\sum_{i=1}^n  g_i(B_i^{k-1,t}) $. Thus, the local updates in Line 8 of Algorithm \ref{alg-fl} are equivalent to 
\begin{align}\label{eq-B-6}
B_i^{k,t+1}&= B_i^{k,t} -\eta \Big({g_i(B_i^{k,t})}\\\notag
&+ (P^k)^T P^{k-1} \frac{1}{\tau}\sum_{t=0}^{\tau-1} \left(\bar{g}(B_i^{k-1,t}) - g_i(B_i^{k-1,t}) \right)\Big), 
\end{align}   
where \( (P^k)^T P^{k-1} \) transforms  
$\frac{1}{\tau} \sum_{t=0}^{\tau-1} ( \bar{g}(B_i^{k-1,t}) - g_i(B_i^{k-1,t}) )$ 
from the previous subspace of \( (P^{k-1})^T \) to the current subspace of \( (P^k)^T \), ensuring that the updates \( B_i^{k,t} \) for  \( t \in [\tau]-1 \) stay in the current subspace of $(P^k)^T$.

\begin{remark}\label{remark-scaffold}
In \eqref{eq-B-6}, we use the idea of variance reduction to mitigate client drift. Specifically, we use the averaged global gradient from the previous round, \( \frac{1}{\tau} \sum_{t=0}^{\tau-1} \bar{g}(B_i^{k-1,t}) \), and replace the previous local gradient \( g_i(B_i^{k-1,t}) \) with the current one \( g_i(B_i^{k,t}) \). Intuitively, each client $i$ uses approximate average gradients during local updates, thereby mitigating the client drift.  
With \(P^k = \mathrm{I}\), the similar variance-reduction idea is employed in full-space methods such as Scaffold and FedLin \cite{karimireddy2020scaffold,FedLin}. To rigorously establish this connection, we provide Proposition~\ref{prop:scaffold} together with a complete proof in Section~\ref{sec:scaffold-equivalence} of the Supplementary  Material, showing that under the full-space setting, FedSub is algebraically equivalent to Scaffold and generates identical iterates at every step.

However, these full-space methods—akin to a direct implementation of \eqref{eq-B-6}—require clients to transmit the local gradients $\frac{1}{\tau} \sum_{t=0}^{\tau-1} g_i(B_i^{k-1,t}) $ in each round $k$ to obtain \( \frac{1}{\tau} \sum_{t=0}^{\tau-1} \bar{g}(B_i^{k-1,t}) \), which introduces additional communication overhead.  To avoid this, we propose a communication-efficient implementation based on dual variables that is equivalent to \eqref{eq-B-6}; see Line~8 of Algorithm~\ref{alg-fl}. 
Through our derivation in Appendix \ref{sec-simplify} (see \eqref{eq-def-c}), we show that these dual variables implicitly perform variance reduction, similar in effect to full-space correction methods such as Scaffold and FedLin, but without incurring additional communication costs.
\end{remark}

\subsection{Subspace Projection to Improve Efficiency}
{Subspace projection} simultaneously reduces the communication, computation, and memory cost of Algorithm~\ref{alg-fl}, as summarized in Table~\ref{tab}. Since the overall cost of DNN training can be obtained by extending the analysis of a single layer, we focus on one layer for clarity. The detailed efficiency
analysis is in Appendix \ref{app-eff-analysis}. 

Communication savings arise because client $i$ sends the low-dimensional variable $B_i^{k,\tau} \in \mathbb{R}^{r\times d}$,  instead of a full-space variable in $\mathbb{R}^{m\times d}$. {Computation and memory savings} result from two factors: 
{\bf (i)} Both the $\tau$ local updates and the dual-variable update are performed  within the low-dimensional subspace; 
{\bf (ii)} the projected stochastic gradient $(P^k)^{T} \hat{\nabla} f_i(x^k + P^k B_{i}^{k,t}) \in \mathbb{R}^{r\times d}$, which is the stochastic gradient of $f_i(x^k + P^k B_i)$ with respect to $B_i$ at $B_i^{k,t}$, can be computed directly without evaluating the full-space stochastic gradient $\hat{\nabla} f_i(x^k + P^k B_{i}^{k,t}) \in \mathbb{R}^{m\times d}$. This reduces both computation and memory, as computing the $r\times d$ gradient costs $\mathcal{O}( C_g(r d))$ and its memory---dominated by intermediate quantities such as activations during backpropagation---is reduced to $\mathcal{O}(M_g(r d))$. As shown in Table \ref{tab}, compared to FedSub-$P^k = \mathrm{I}$, FedAvg, and Scaffold, although FedSub introduces additional matrix multiplications with a cost of $\mathcal{O}(\tau m r d + 2 m r d)$, it reduces the complexity of gradient evaluation from $C_g(md)$ to $C_g(rd)$. Since gradient evaluation typically dominates the training cost of DNNs, the overall computation cost is still reduced.

\section{Analysis}
In our analysis, we introduce the following assumptions. 
\begin{assumption}\label{asm-smooth}
Each function $f_i(x)$ is $L_f$-smooth such that 
\begin{equation*}
\|\nabla f_i(x)-\nabla f_i(z)\|\le L_f\|x-z\|, ~\forall i, x,z.    
\end{equation*}
\end{assumption}

\begin{assumption}\label{asm-sgd-var}
Let \( \hat{\nabla} f_i(x) \) denote the stochastic gradient computed using a mini-batch of size \( b \). We assume that the variance of the stochastic gradient is bounded as follows:
\[
\mathbb{E}\left[\|\hat{\nabla} f_i(x) - \nabla f_i(x)\|^2\right] \le \frac{\sigma^2}{b}, \quad \forall i,
\]
where the expectation $\mathbb{E}$ is taken over the randomness of the mini-batch sampling.
\end{assumption}

\begin{assumption}\label{asm-Bg}
There exists a constant $B_g $ such that
\begin{equation*}
    \frac{1}{n}\sum_{i=1}^n \|\nabla f_i(x^k)\|^2 \leq B_g^2, 
    \quad \forall\, k,
\end{equation*}
where $x^k$ is the iterate generated by Algorithm~\ref{alg-fl} 
at communication round $k$. The inequality is required to hold for every iterate generated by the algorithm.
\end{assumption}

Assumptions \ref{asm-smooth} and \ref{asm-sgd-var} are standard in the analysis of nonconvex stochastic gradient algorithms. Assumption \ref{asm-Bg} is introduced specifically to handle the additional variance introduced by the random subspace projection matrices $P^k$. More specifically, within our analysis, Assumption~\ref{asm-Bg} is invoked only to bound the projection-induced term analyzed in Section~\ref{supp-lem2} of the Supplementary Material. Notably, Assumption \ref{asm-Bg} becomes unnecessary in the full-space case when $P^k = \mathrm{I}$.

\begin{remark}\label{remark-Bg}
Assumption~\ref{asm-Bg} can be guaranteed under a uniform bounded-iterate condition. Specifically, suppose that there exists a deterministic bounded constant $D$ such that $\|x^k\|$$\le D$ for all $k$ and for every realization of the algorithmic randomness. Then, by Assumption~\ref{asm-smooth}, $\|\nabla f_i(x^k)\|$ $\le\|\nabla f_i(x^0)\|$$+L_f\|x^k-x^0\|$ $\le\|\nabla f_i(x^0)\|+2L_fD$. Consequently, $\frac1n\sum_{i=1}^n\|\nabla f_i(x^k)\|^2\le$$(\max_i\|\nabla f_i(x^0)\|$$+2L_fD)^2$, and hence Assumption~\ref{asm-Bg} holds with $B_g=\max_i\|\nabla f_i(x^0)\|$ $+2L_fD$. Such a uniform bounded-iterate condition is also used in existing convergence analyses, for example, in GaLore~\cite[Theorem~3.8]{zhao2024galore}. In practice, training iterates are also commonly observed to remain within a bounded region over the considered optimization horizon.
\end{remark}

\begin{theorem}\label{thm}
Under Assumptions \ref{asm-p}-\ref{asm-Bg}, for Algorithm \ref{alg-fl}, if $\eta \le \mathcal{O}\left(\frac{1}{\theta_r\theta_m^2 L_f^2 n \tau}\right)$,  we have 
\begin{equation*}
\begin{aligned}
&\frac{1}{K}\sum_{k=0}^{K-1} \bE\|\nabla f(x^k)\|^2\\
&\le \mathcal{O}\left( \frac{\theta_m}{K  \tau \eta} + {\eta {\theta_r^2} \theta_m^3 {\frac{1}{n}} \frac{\sigma^2}{b}   }+ \eta \tau {\theta_r^2} \theta_m^2  \left(\theta_m-1 \right) {\frac{1}{n}} B_g^2 \right). 
\end{aligned}
\end{equation*}
Further, if $\eta \le \mathcal{O}\left(\frac{1}{\theta_r\theta_m^2 L_f^2 n \tau}\right)$ and  $\eta =\mathcal{O}\left(\frac{\sqrt{n}}{  \sqrt{ K\tau }} \right)$, we have 
\begin{equation*}
\begin{aligned}
&\frac{1}{K}\sum_{k=0}^{K-1} \bE\|\nabla f(x^k)\|^2\\
&\le \mathcal{O}\left( \frac{\theta_m}{\sqrt{K \tau n}}  + {  \frac{{\theta_r^2}\theta_m^3}{\sqrt{K \tau n}} \frac{\sigma^2}{b}   } + \frac{\sqrt{\tau} {\theta_r^2}\theta_m^2  \left(\theta_m-1 \right)}{\sqrt{K n}}  B_g^2\right).   
\end{aligned}
\end{equation*}
\end{theorem}
\begin{proof}
See Appendix \ref{app-proof-thm}.    
\end{proof}

According to Theorem~\ref{thm}, we have the following observations:
\textbf{(i)} The error consists of three parts: the first corresponds to sublinear convergence, the second is due to the variance of stochastic gradients, and the third arises from the variance of $P^k(P^k)^T$.   When $ \eta = \mathcal{O}( \frac{\sqrt{n}}{\sqrt{K\tau}} )$, the error is  bounded by 
$\mathcal{O}( \frac{1}{\sqrt{K\tau n}} + \frac{\sqrt{\tau}}{\sqrt{K n}} )$. 
Since
$ \theta_r = \max_l \frac{r_l}{m_l} $ and $\theta_m = \max_l \frac{m_l}{r_l} $, we take $\theta_r \theta_m = \mathcal{O}( 1)$. 
\textbf{(ii)} When $P^k = \mathrm{I}$, we have $\theta_r = \theta_m = 1$, and the error simplifies to 
$\mathcal{O}( \frac{1}{\sqrt{K\tau n}})$, 
highlighting the benefit of increasing the number of local updates $\tau$. 
However, overly large $\tau$  may lead to 
$\eta=\mathcal{O}(\frac{\sqrt{n}}{\sqrt{K\tau}}) > \mathcal{O}(\frac{1}{n\tau})$, 
which violates the required condition that $\eta \le \mathcal{O}\left(\frac{1}{\theta_r\theta_m^2 L_f^2 n \tau}\right)$, making the theorem inapplicable.
\textbf{(iii)} When $P^k = \mathrm{I}$, $n = 1$, and $\tau = 1$, the result reduces to 
$\mathcal{O}( \frac{1}{\sqrt{K}} )$, 
which recovers the optimal convergence rate of the single-machine stochastic gradient descent method \cite[Theorem 1]{reddi2016stochastic}.

\begin{remark}\label{remark-r}
Assume a single layer for simplicity. Theorem~\ref{thm} explicitly characterizes the efficiency--accuracy trade-off induced by the subspace dimension \(r\). A smaller \(r\) reduces the communication, computation, and memory requirements, whose dominant complexities are proportional to \(rd\), \(C_g(rd)\), and \(M_g(rd)\), respectively, while increasing the projection-related term
$\mathcal{O}\left(\eta\left({m}/{r}-1\right)B_g^2\right).$ Under the standard diminishing-stepsize choice \(\eta=\mathcal{O}(1/\sqrt{K})\) and the fact $m/r\le m$, this term decreases at the rate \(\mathcal{O}(1/\sqrt{K})\) and therefore vanishes as \(K\to\infty\). Consequently, a smaller $r$ yields higher efficiency at the expense of a larger transient approximation error during optimization. 

{In practice, the choice of $r$ should be guided by the desired balance between efficiency and accuracy. When accuracy is the primary objective, a larger $r$ is preferable. When communication, computation, or memory efficiency is of greater concern, a smaller $r$ can be adopted to achieve larger efficiency, with a smaller stepsize required to guarantee the level of accuracy. 
}\end{remark}

\begin{figure*}[t!]
  \centering
  \begin{subfigure}[b]{0.325\textwidth}
    \includegraphics[width=\textwidth]{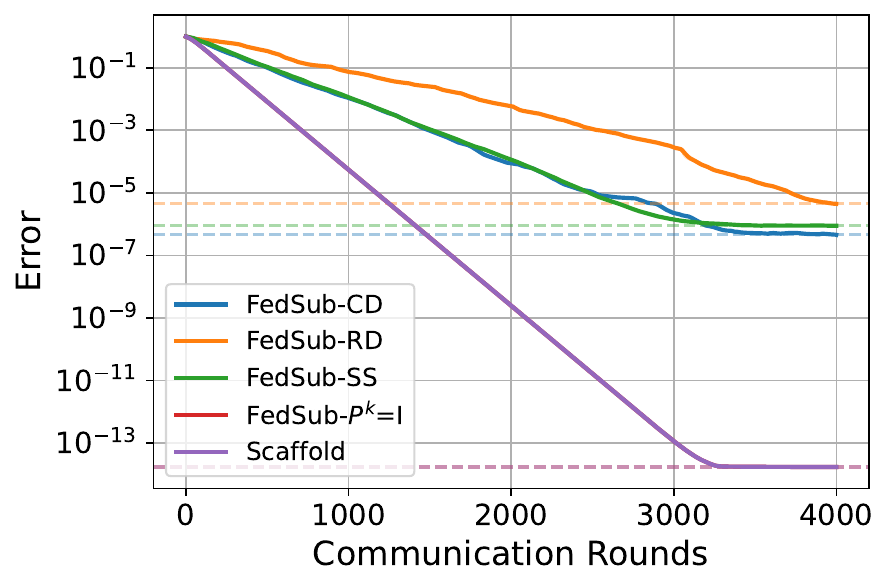}
    \caption{Impact of $P^k$}
    \label{fig:log1}
  \end{subfigure}
  \hfill
  \begin{subfigure}[b]{0.325\textwidth}
    \includegraphics[width=\textwidth]{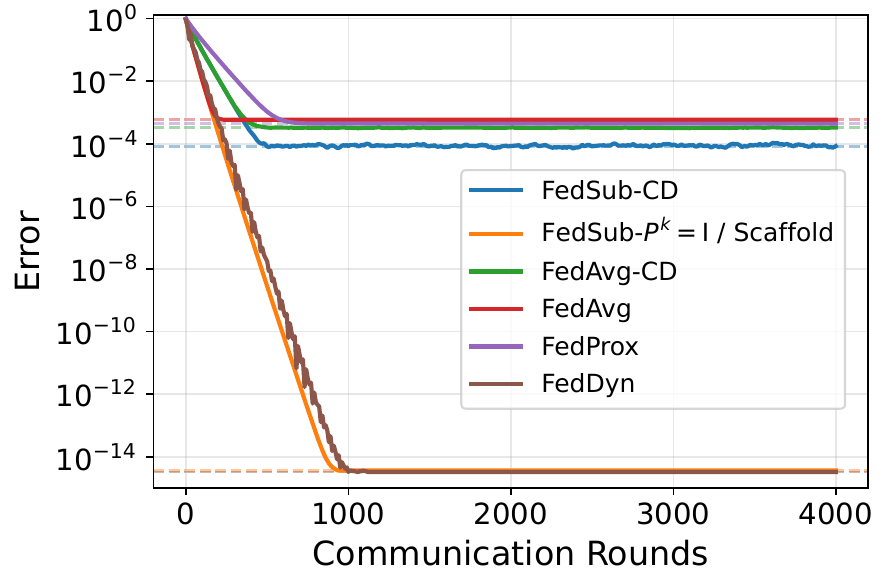}
    \caption{Comparison with baselines}
    \label{fig:log2}
  \end{subfigure}
  \hfill
  \begin{subfigure}[b]{0.325\textwidth}
    \includegraphics[width=\textwidth]{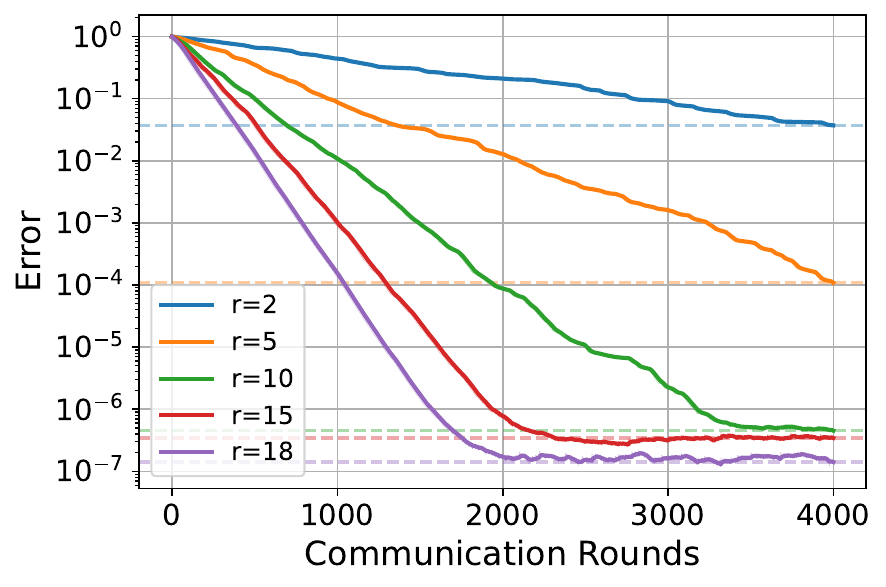}
    \caption{Impact of subspace dimensions $r$}
    \label{fig:log3}
  \end{subfigure}

  \caption{Logistic regression: $\text{Error} = {\|x - x^{\star}\|}/{\|x^{\star}\|}$ versus communication rounds, where \(x^{\star}\) is the precomputed optimal solution to \eqref{eq-lr}. When measured by communication rounds, FedSub-$P^k = \mathrm{I}$ exhibits identical performance to Scaffold.}
  \label{fig:linear}
\end{figure*}

\section{Numerical Experiments}
We conduct a comprehensive evaluation of the proposed algorithm to assess its efficiency in terms of computation, communication, and memory usage. We compare our method with several established benchmarks across a diverse set of tasks, including logistic regression, computer vision, and natural language processing. The source code for all experiments is available online.\footnote{\url{https://github.com/JimmyAwoe/FedSub}}

There exist multiple approaches to generate a projection matrix \(P^k \in \mathbb{R}^{m \times r}\) that exactly or approximately satisfies Assumption~\ref{asm-p}~\cite{he2024subspace, kozak2023zeroth}. In our experiments, we consider three commonly used examples:

\textit{Coordinate Descent (CD)}: Randomly select \(r\) columns from the identity matrix \(I\), and multiply each selected column by either \(\sqrt{{m}/{r}}\) or \(-\sqrt{{m}/{r}}\) with equal probability.

\textit{Random Generation (RD)}: Sample each entry of \(P^k\) independently from a Gaussian distribution \(\mathcal{N}(0, {1}/{r})\). 

\textit{Spherical Smoothing (SS)}: Generate a standard  Gaussian random matrix and perform $QR$ decomposition to obtain factors $Q$ and $R$. By adjusting the signs of the columns of $Q$ accordingly, we ensure that the diagonal entries of $R$ are positive. Then, set \(P^k\) as the first \(r\) columns of the orthogonal matrix \(Q\), scaled by \(\sqrt{{m}/{r}}\).

For each layer $l$ and communication round $k$, a shared random seed $s_l^k$ is broadcast to all clients, and each client then generates the same projection matrix $P_l^k$ locally. The seeds $\{s_l^k\}$ are generated independently across $l$ and $k$.

\subsection{Logistic Regression }\label{sec-experiment-lr}

We consider the logistic regression problem
\begin{align}\label{eq-lr}
\min_{x\in \mathbb{R}^{m}}\frac{1}{n} \sum_{i=1}^{n} f_i(x) 
+\frac{\rho}{2}\|x\|^2,
\end{align}
and each local loss $f_i(x)$ is defined by
$$f_i(x)= -\frac{1}{J} \sum_{j=1}^J \left[y_{i,j} \log \left(\hat{y}_{i,j}\right)+\left(1-y_{i,j}\right) \log \left(1-\hat{y}_{i,j}\right)\right], $$
where 
$\hat{y}_{i,j}=\sigma\left(a_{i,j}^T x \right)$, $\sigma(a)={1}/{(1+e^{-a})}$ for a given scalar $a$,  $a_{i,j} \in \mathbb{R}^m $ represents feature of $j$-th sample on client $i$ and  \(y_{i,j}\) is the corresponding label.

We synthetically generated \( 60,000\) (i.e., $nJ=60,000$) data samples, each with \(m = 20\) dimensions, along with their corresponding labels. These data samples are partitioned into 30 clusters, where each cluster is centered around a distinct hyperplane with a small amount of noise, thereby introducing heterogeneity across clusters. 
We set the number of clients to \(n = 30\), such that each client holds the data from one cluster. 

Each experiment is repeated five times, and we report the average results. To eliminate the influence of stochastic gradients, we use full gradients throughout the experiments of logistic regression.



\subsubsection{Impacts of Subspace Projection Matrix $P^k$}

In the first set of experiments, we investigate the impacts of the subspace projection matrices \( P^k \) on Algorithm \ref{alg-fl}. 
We construct \( P^k\in \mathbb{R}^{m\times r} \) in Algorithm~\ref{alg-fl} using three different methods: CD, SS, and RD, resulting in the algorithms {FedSub-}CD, {FedSub-}SS, and {FedSub-}RD, respectively.   
As a baseline, we also consider the case where \( P^k = \mathrm{I} \), denoted as {{FedSub-}\( P^k = \mathrm{I} \)}.   We set $\rho=0.01$, $\eta=0.2$, $r=10$ and $\tau=5$.

As shown in Fig.~\ref{fig:log1}, FedSub-\(P^k = \mathrm{I}\), which performs in the full-dimensional space, achieves exact convergence. In contrast, FedSub-CD, FedSub-SS, and FedSub-RD operate in low-dimensional subspaces, and the projection inevitably introduces a performance gap. This observation is consistent with Theorem~\ref{thm}, where the projection error is characterized by the term \(\theta_m-1\). Since \(\theta_m-1 \neq 0\) for these subspace projection schemes, a nonzero error is unavoidable. Among the three projection strategies, FedSub-CD consistently achieves the best overall performance. Therefore, we adopt CD as the default projection strategy in all subsequent experiments. 

Furthermore, when performance is measured by communication rounds, FedSub-\(P^k = \mathrm{I}\) exhibits identical convergence behavior to Scaffold~\cite{karimireddy2020scaffold}. This observation is consistent with Proposition~\ref{prop:scaffold} in Section~\ref{sec:scaffold-equivalence} of the Supplementary Material, where we establish the theoretical equivalence between the two methods. Since Scaffold requires communicating additional correction variables between the server and clients, FedSub-\(P^k = \mathrm{I}\) achieves the same convergence behavior with lower communication overhead.


\subsubsection{Comparison of Algorithm \ref{alg-fl} with Baselines}

Next, we compare Algorithm~\ref{alg-fl} with several baseline methods to demonstrate the effectiveness of the proposed dual-variable correction in mitigating client drift. Specifically, we consider six algorithms: FedSub-CD, FedSub-\(P^k = \mathrm{I}\), FedAvg, FedAvg-CD,  FedProx~\cite{li2020federated}, and FedDyn~\cite{rev2-feddyn}. Among them, FedProx and FedDyn are representative full-dimensional methods designed to alleviate or eliminate client drift.  FedAvg-CD denotes a variant of Algorithm~\ref{alg-fl} in which all dual variables are removed, i.e., \(\Lambda_i^k=0\) for all \(k\), while the projection matrix \(P^k\) is constructed using the CD strategy. For FedProx and FedDyn, the regularization parameters are set to \(3\) and \(0.03\), respectively. Unless otherwise specified, we set \(\rho=0.01\), \(\eta=0.2\), \(r=10\), and \(\tau=5\).

As shown in Fig.~\ref{fig:log2}, among the full-dimensional algorithms, both our FedSub-\(P^k = \mathrm{I}\) and FedDyn completely eliminate the client-drift error, achieving a final accuracy of approximately \(10^{-14}\). However, FedDyn requires solving a regularized local subproblem accurately at every communication round, resulting in higher computational cost. In contrast, both FedAvg and FedProx converge only to an accuracy slightly better than \(10^{-3}\). Although the proximal regularization in FedProx partially mitigates client drift and yields a modest improvement over FedAvg, it does not eliminate the client-drift error completely.

For the subspace methods, both FedSub-CD and FedAvg-CD perform local updates in low-dimensional subspaces. However, by incorporating dual-variable correction, FedSub-CD achieves an accuracy of approximately \(10^{-4}\), whereas FedAvg-CD converges only to around \(10^{-3}\). Interestingly, FedAvg-CD slightly outperforms FedAvg. A possible explanation is that restricting local updates to a low-dimensional subspace reduces the divergence among client updates, thereby partially mitigating client drift.


These results demonstrate that, whether in full-space or subspace settings, incorporating dual variables effectively mitigates client drift and thus improves accuracy.

\begin{figure*}[t]
  \centering
  \begin{subfigure}[b]{0.48\textwidth}
    \centering
    \includegraphics[width=\textwidth]{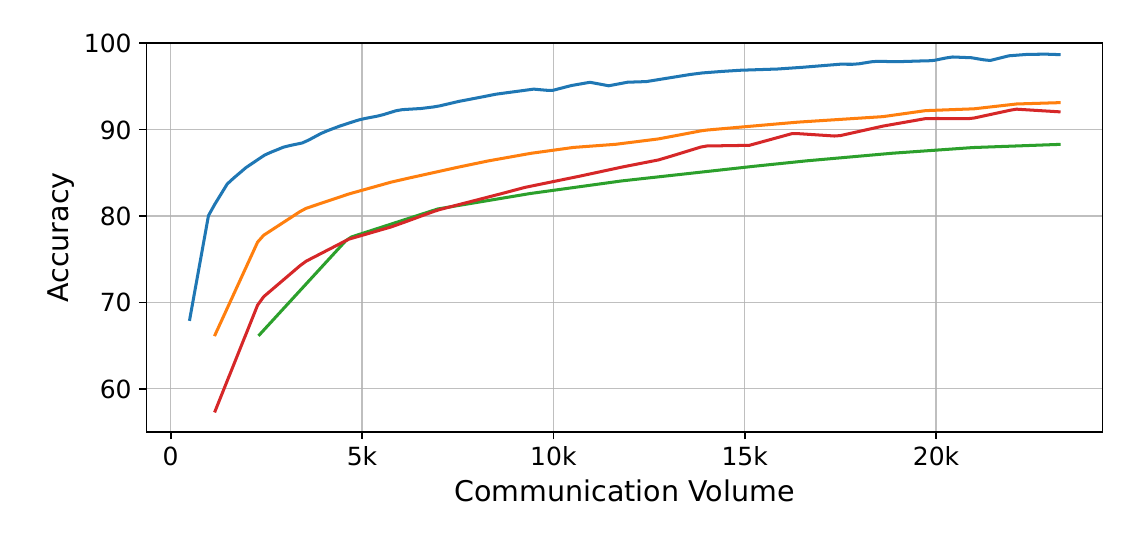}
    \caption{Training}
  \end{subfigure}
  \hfill
  \begin{subfigure}[b]{0.48\textwidth}
    \centering
    \includegraphics[width=\textwidth]{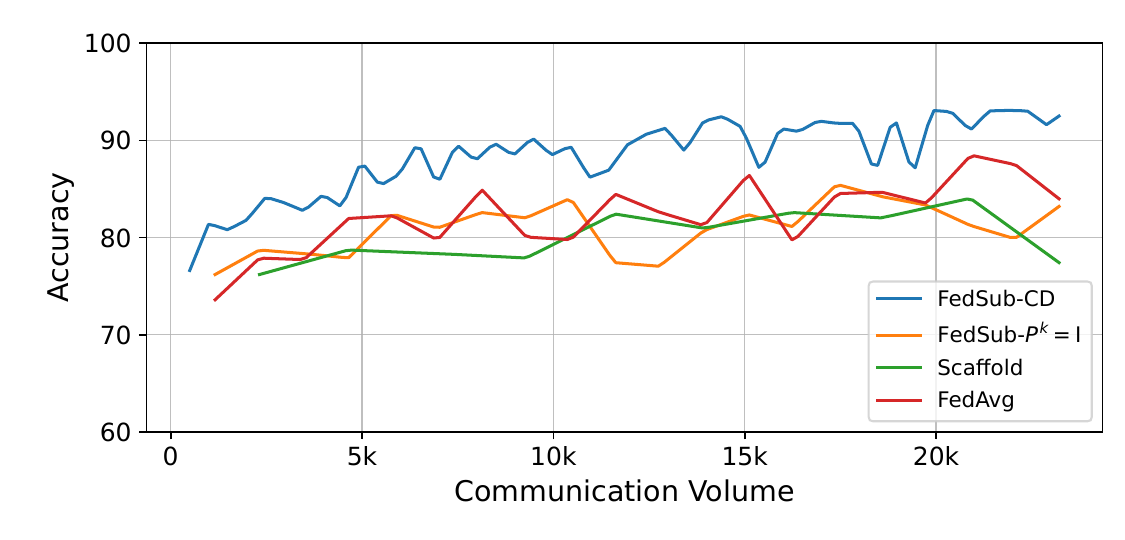}
      \caption{Test}
  \end{subfigure}
  \caption{Computer vision: Top-1 accuracy on the CIFAR-10 training set (a) and test set (b).}
  \label{fig:resnet-cifar10}
\end{figure*}

\begin{figure*}[t]
  \centering
  \begin{subfigure}[b]{0.48\textwidth}
    \centering
    \includegraphics[width=\textwidth]{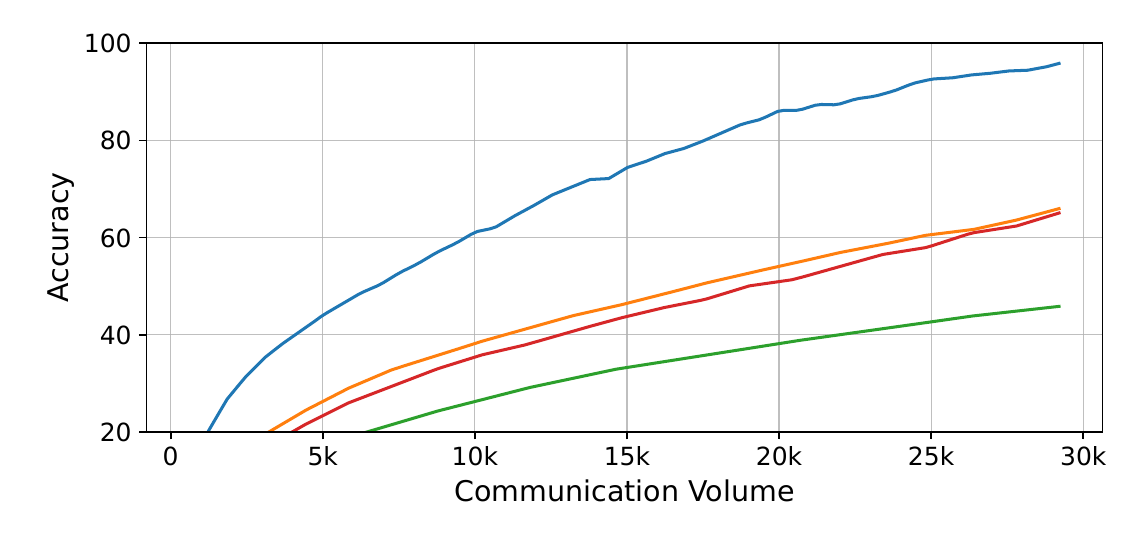}
      \caption{Training}
  \end{subfigure}
  \hfill
  \begin{subfigure}[b]{0.48\textwidth}
    \centering
    \includegraphics[width=\textwidth]{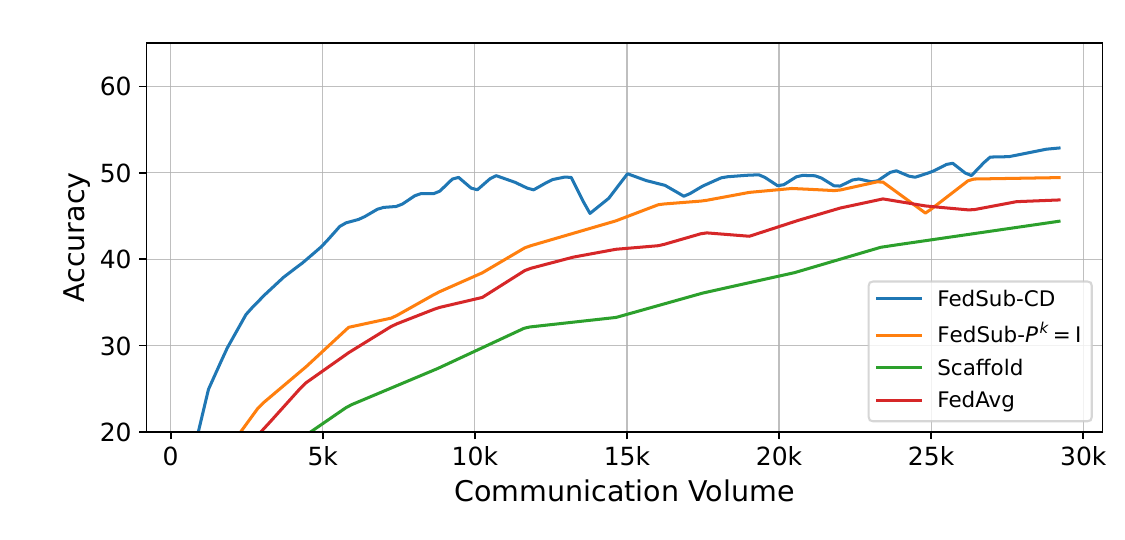}
      \caption{Test}
  \end{subfigure}
  \caption{Computer vision: Top-1 accuracy on the CIFAR-100 training set (a) and test set (b).}
  \label{fig:resnet}
\end{figure*}

\subsubsection{Impact of Number of Subspace Dimensions $r$} Finally, we investigate the performance of the {FedSub-}CD algorithm under different subspace dimensions $r$. We set $\rho=0.01$, $\eta=0.2$, and $\tau=5$. 
As shown in Fig. \ref{fig:log3}, the accuracy of our FedSub improves with increasing 
$r$, which aligns with Theorem \ref{thm}.

\subsection{Computer Vision}

In the first set of experiments on computer vision, we evaluate the performance of FedSub on computer vision tasks using the CIFAR-10 and CIFAR-100  datasets \cite{krizhevsky2009learning}. To simulate a data heterogeneity scenario, different image categories are assigned to different clients. We construct an 110-layer ResNet model \cite{he2016deep} and use our FedSub on the convolutional layers, thereby reducing the number of trainable parameters to 43\% of the original. The detailed hyperparameter settings are provided in Table~\ref{tab:hparams} of the Supplementary Material.

As shown in Fig.~\ref{fig:resnet}, when performance is measured in terms of communication volume, Scaffold requires approximately twice the communication of FedSub-\(P^k = \mathrm{I}\) to achieve comparable accuracy. Moreover, FedSub-CD achieves the best training and testing accuracy on both CIFAR-10 and CIFAR-100. These results demonstrate that, despite substantially reducing communication, computation, and memory costs through subspace projections (see Table~\ref{tab}), FedSub-CD maintains competitive performance on both image classification tasks.

\begin{figure*}[t]
    \centering
    \begin{subfigure}[b]{0.48\textwidth}
        \centering
        \includegraphics[width=\textwidth]{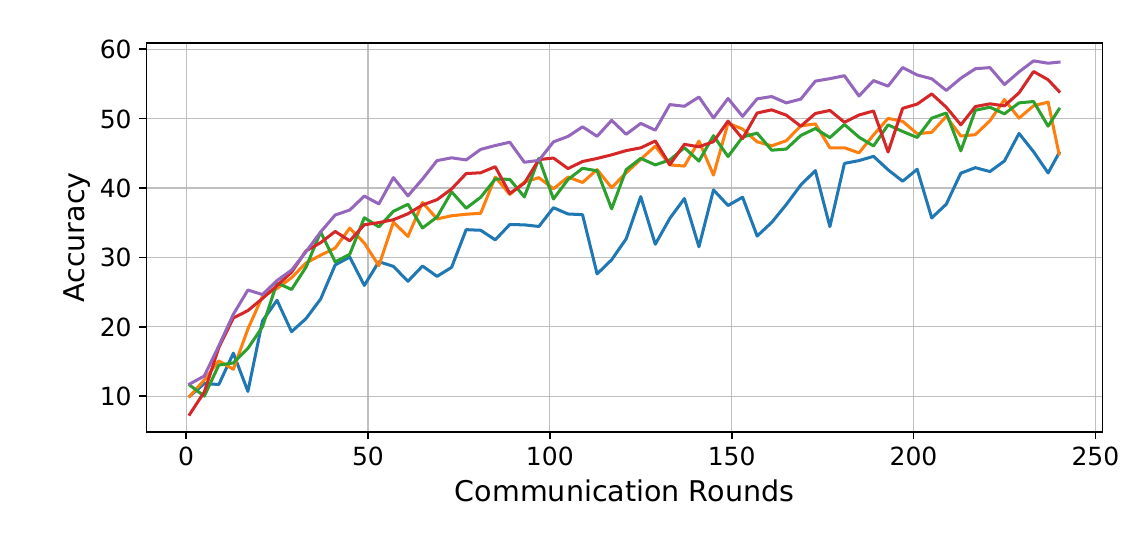}
    \end{subfigure}
    \hfill
    \begin{subfigure}[b]{0.48\textwidth}
        \centering
        \includegraphics[width=\textwidth]{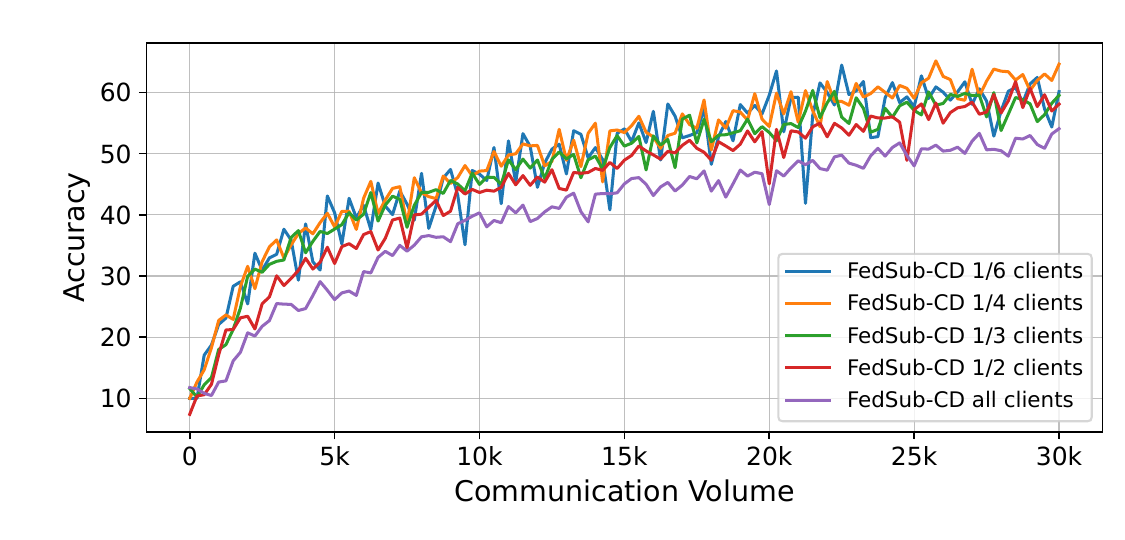}
    \end{subfigure}
    \caption{Computer vision: Top-1 accuracy of FedSub-CD under partial participation and non-uniform data sizes on CIFAR-10.}
    \label{fig:partial-participation}
\end{figure*}
In the second set of computer vision experiments, we evaluate FedSub-CD under partial participation and non-uniform client data sizes. We use ResNet110 on CIFAR-10 with 12 clients, whose local training sample sizes are 7,000, 6,500, 6,000, 5,500, 5,000, 4,500, 4,000, 3,500, 3,000, 2,500, 1,500, and 1,000, respectively. Each client has three dominant labels accounting for approximately 75\% of its local data. We consider participation rates of \(1/6\), \(1/4\), \(1/3\), \(1/2\), and \(1\), with the remaining hyperparameters summarized in Table~\ref{tab:hparams} of the Supplementary Material.

As shown in Fig.~\ref{fig:partial-participation}, full participation achieves the highest test accuracy when performance is measured by communication rounds. However, when performance is evaluated under a fixed server-received communication volume, a moderate participation rate of \(1/4\) achieves the best accuracy. This result suggests that the participation rate needs to be properly selected: a higher participation rate allows more clients to contribute in each round but requires more communication per round, while a lower participation rate reduces communication cost per round but limits the amount of client information incorporated into each update.

\begin{figure*}[htbp]
  \centering
  \begin{subfigure}[b]{0.325\textwidth}
    \includegraphics[width=\textwidth]{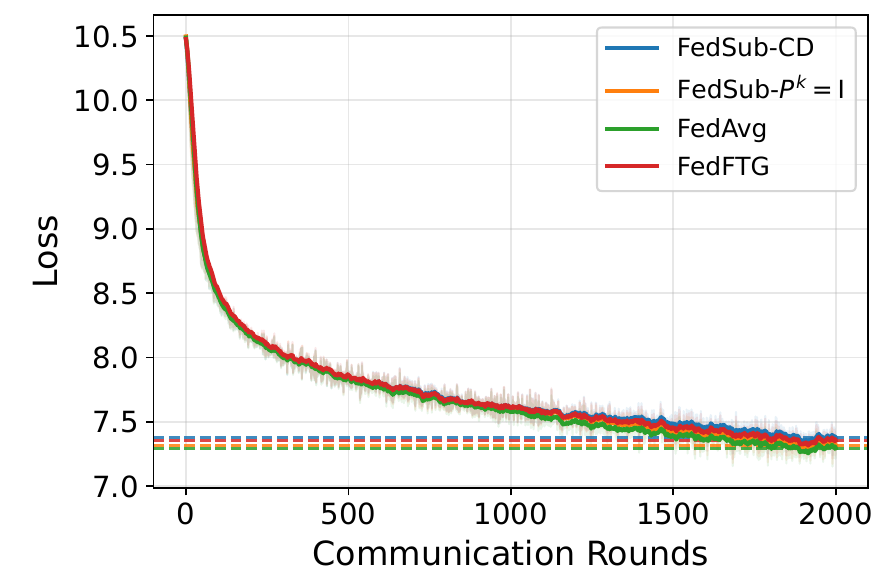}
    \caption{Pre-training loss}
    \label{fig:pretrain}
  \end{subfigure}
  \hfill
  \begin{subfigure}[b]{0.325\textwidth}
    \includegraphics[width=\textwidth]{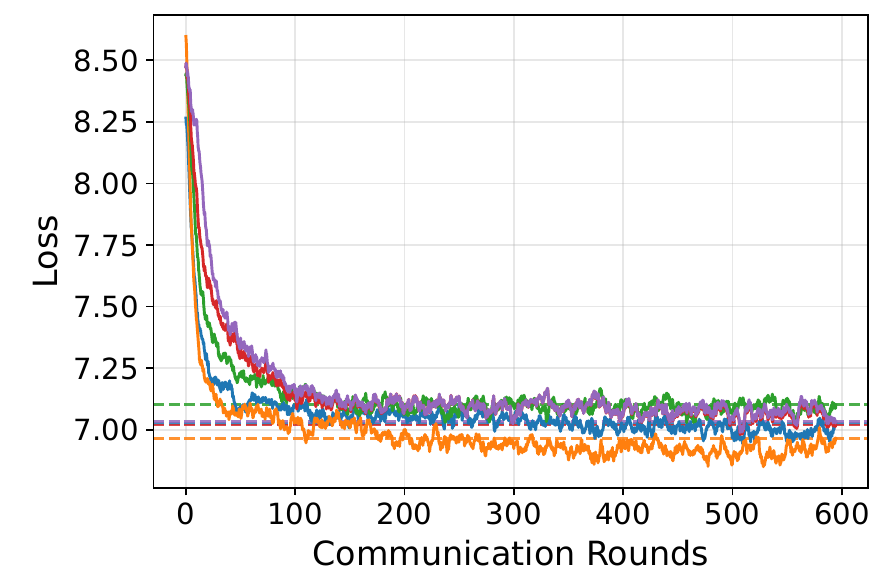}
    \caption{Fine-tuning train loss}
    \label{fig:finetune}
  \end{subfigure}
  \hfill
  \begin{subfigure}[b]{0.325\textwidth}
    \includegraphics[width=\textwidth]{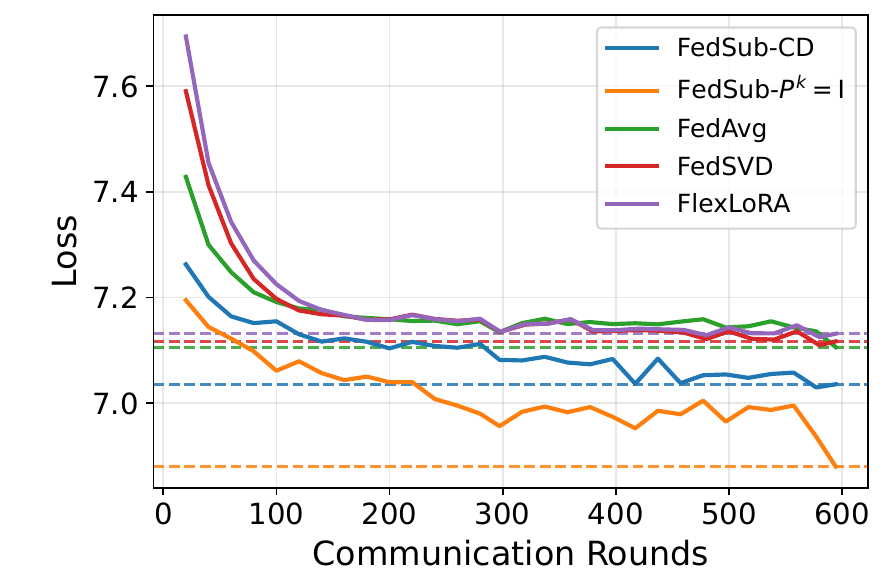}
    \caption{Fine-tuning test loss}
    \label{fig:finetuneeval}
  \end{subfigure}
    \caption{NLP: (a) and (b) show exponentially weighted averages with a smoothing factor of 0.05.}
    \label{fig:nlp}
\end{figure*}

\subsection{Natural Language Processing}

We evaluate the performance of the proposed algorithm on natural language processing (NLP) tasks within the standard LLMs training framework, including pre-training and fine-tuning. Further, we assess its memory efficiency, communication efficiency, and runtime performance. Besides FedAvg and Scaffold, we also compare with state-of-the-art subspace federated learning algorithms for NLP training, namely FlexLoRA \cite{sub-bai2024federated}, FedSVD \cite{fedlora-1-lee2025fedsvd}, and FedFTG \cite{sub-mahla2024exploring}. For federated fine-tuning, FlexLoRA and FedSVD are two representative LoRA-based methods. FlexLoRA represents the class of FedAvg-based LoRA-style methods and supports heterogeneous ranks across clients. FedSVD introduces an SVD-based reparameterization strategy to adapt the LoRA representation over time based on aggregated client updates, improving the effectiveness of federated fine-tuning. FedFTG, on the other hand, combines FedAvg with GaLore to better adapt to pre-training tasks.

\begin{table*}[htbp]
    \centering
    \caption{NLP fine-tuning: 
    This table shows the number of trainable parameters and the communicated parameters (in parentheses). {FedSub-}CD significantly reduces both total trainable parameters (e.g., LLaMA-1B with \(r = 256\) drops to 261.19M vs FedAvg/{FedSub-}\(P^k = \mathrm{I}\) at 1339.08M) and communicated parameters (e.g., LLaMA-350M with \(r = 256\) drops to 65.08M vs FedAvg at 367.97M). }
    \begin{tabularx}{\linewidth}{l *{9}{X}} 
        \toprule 
        \multirow{2}{*}{\textbf{Model}} & \multicolumn{3}{c}{\textbf{LLaMA-60M}} & \multicolumn{3}{c}{\textbf{LLaMA-350M}} & \multicolumn{3}{c}{\textbf{LLaMA-1B}} \\
        \cmidrule(lr){2-4} \cmidrule(lr){5-7} \cmidrule(lr){8-10} 
        & $r=$ 64 & $r=$  128 & $r=$  256 &$r=$  64 & $r=$  128 & $r=$  256 & $r=$  64 & $r=$  128 & $r=$  256 \\
        \midrule 
        \textbf{{FedSub-}CD}
        & $\underset{\text{(2.72M)}}{\text{35.50M}}$
        & $\underset{\text{(5.44M)}}{\text{38.22M}}$ 
        & $\underset{\text{(10.88M)}}{\text{43.66M}}$ 
        & $\underset{\text{(16.27M)}}{\text{81.86M}}$ 
        & $\underset{\text{(32.54M)}}{\text{98.12M}}$ 
        & $\underset{\text{(65.08M)}}{\text{130.66M}}$ 
        & $\underset{\text{(32.50M)}}{\text{163.68M}}$ 
        & $\underset{\text{(65.01M)}}{\text{196.18M}}$ 
        & $\underset{\text{(130.02M)}}{\text{261.19M}}$ \\
        {FedSub-}$P^k = \mathrm{I}$ 
        & \multicolumn{3}{c}{58.07M (25.30M)} 
        & \multicolumn{3}{c}{367.97M (302.38M)} 
        & \multicolumn{3}{c}{1339.08M (1207.91M)} \\
        FedAvg \cite{fedavg}
        & \multicolumn{3}{c}{58.07M (58.07M)} 
        & \multicolumn{3}{c}{367.97M (367.97M)} 
        & \multicolumn{3}{c}{1339.08M (1339.08M)} \\
        {FlexLoRA \cite{sub-bai2024federated} }
        & $\underset{\text{(5.00M)}}{\text{37.77M}}$
        & $\underset{\text{(9.99M)}}{\text{42.77M}}$
        & $\underset{\text{(19.99M)}}{\text{52.77M}}$
        & $\underset{\text{(29.91M)}}{\text{95.50M}}$
        & $\underset{\text{(59.82M)}}{\text{125.40M}}$
        & $\underset{\text{(119.64M)}}{\text{185.22M}}$
        & $\underset{\text{(59.77M)}}{\text{190.94M}}$
        & $\underset{\text{(119.53M)}}{\text{250.71M}}$
        & $\underset{\text{(239.07M)}}{\text{370.24M}}$ \\
          FedSVD \cite{fedlora-1-lee2025fedsvd}
        & {$\underset{\text{(2.72M)}}{\text{35.50M}}$}
        & {$\underset{\text{(5.44M)}}{\text{38.22M}}$}
        & {$\underset{\text{(10.88M)}}{\text{43.66M}}$}
        & {$\underset{\text{(16.27M)}}{\text{81.86M}}$}
        & {$\underset{\text{(32.54M)}}{\text{98.12M}}$}
        & {$\underset{\text{(65.08M)}}{\text{130.66M}}$}
        & {$\underset{\text{(32.50M)}}{\text{163.68M}}$}
        & {$\underset{\text{(65.01M)}}{\text{196.18M}}$}
        & {$\underset{\text{(130.02M)}}{\text{261.19M}}$} \\
        {FedFTG \cite{sub-mahla2024exploring}}
        & \multicolumn{3}{c}{58.07M (58.07M)} 
        & \multicolumn{3}{c}{367.97M (367.97M)} 
        & \multicolumn{3}{c}{1339.08M (1339.08M)} \\
        \bottomrule 
    \end{tabularx}
    \label{tab:params_stat}
\end{table*}

\subsubsection{Pre-training}

We use LLaMA-60M~\cite{touvron2023LLaMA} as the base model and train it on the C4 dataset~\cite{raffel2020exploring} using the T5 tokenizer~\cite{raffel2020exploring}. The detailed training settings are listed in Table~\ref{tab:hparams} of the Supplementary Material. To improve training efficiency, we use mixed-precision training (bfloat16), which is a common technique in large-scale language model training. We also apply gradient checkpointing (also known as activation recomputation) to reduce GPU memory usage, making the training setup more representative of real-world conditions.

We apply our FedSub method only to the linear layers in the Attention and Feedforward Network (FFN) modules. The rest of the model is trained using standard stochastic gradient descent (SGD). While compressing the LLaMA-head layer could further reduce model size, it leads to worse performance and is therefore not used. To save memory, we use the same subspace projection matrices for all layers, so only two projection matrices need to be stored. Full training settings are listed in Table~\ref{tab:hparams} of the Supplementary Material.

We compare {FedSub-}CD, {FedSub-}$P^k = \mathrm{I}$, FedAvg, and FedFTG on pre-training tasks.
Following ideas from recent work~\cite{zhao2024galore, he2024subspace}, we use a hybrid approach for {FedSub-}CD: during the first 30\% of training, we periodically compute subspace projection matrices by performing SVD on the gradients of the model parameters (sampled every 10$\tau$ steps), and take the right singular vectors. After this warm-up phase, we switch to the standard CD training method.

As shown in Fig.~\ref{fig:pretrain}, both {FedSub-}CD and {FedSub-}$P^k = \mathrm{I}$ achieve performance comparable to FedAvg and {FedFTG \cite{sub-mahla2024exploring}  in terms of communication rounds}, which demonstrates that our subspace algorithm can attain accuracy similar to that of full-space and existing sub-space methods on the pre-training task. However,  our FedSub-CD is more efficient on communication, computation, and memory overhead, as we will show in the following sections. 

\subsubsection{Fine-tuning}
To further demonstrate the effectiveness of our FedSub on NLP tasks, we extend the experiments to the fine-tuning stage, which is an important step after LLM pre-training. We use the LLaMA-1B model with the T5 tokenizer~\cite{raffel2020exploring} and fine-tune it on the WikiText-2-Raw-v1 dataset~\cite{merity2016pointer}, which contains 36,718 high-quality Wikipedia articles. To improve training speed and reduce memory usage, we use FlashAttention~\cite{dao2022flashattention}, an optimized attention mechanism that accelerates computation. For a fair
comparison, we applied the same layer-wise compression level to all subspace algorithms. Detailed hyperparameter settings are provided in Table~\ref{tab:hparams} of the Supplementary Material.

As shown in Fig.~\ref{fig:finetune}, our {FedSub-}CD and {FedSub-}$P^k = \mathrm{I}$ achieve better loss values compared to FedAvg, FlexLoRA, and FedSVD in fine-tuning tasks in terms of communication rounds. Furthermore, as shown in Fig.~\ref{fig:finetuneeval}, our {FedSub-}CD and {FedSub-}$P^k = \mathrm{I}$ achieve a more pronounced reduction in test loss during fine-tuning, further validating its stronger generalization performance.

As shown in Table~\ref{tab:params_stat}, compared to full-space methods such as FedAvg and our FedSub-$P^k = \mathrm{I}$, our FedSub-CD performs local updates entirely within a low-dimensional subspace and communicates only subspace variables. This substantially reduces both the number of trainable parameters and the communication cost, thereby improving training efficiency. In contrast, although FlexLoRA and FedSVD, which {integrate} FedAvg with LoRA, also {adopt subspace techniques}, most of {their} parameters are frozen, making them suitable only for fine-tuning rather than pre-training. On the other hand, FedFTG, which combines FedAvg with GaLore, merely compresses gradients; its local updates are still carried out in the full space, and the communicated variables also remain full-space. As a result, its trainable parameters and communication overhead are essentially the same as those of FedAvg.

\begin{figure*}[htbp]
    \centering
    \begin{subfigure}[b]{0.48\textwidth}
        \centering
        \includegraphics[width=\textwidth]{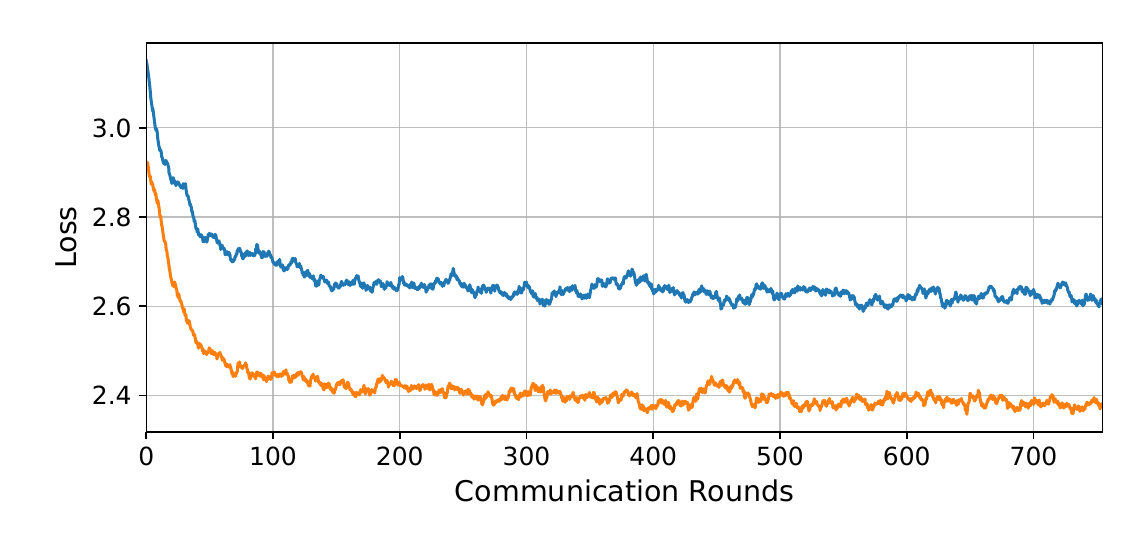}
          \caption{Training}
    \end{subfigure}
    \hfill
    \begin{subfigure}[b]{0.48\textwidth}
        \centering
        \includegraphics[width=\textwidth]{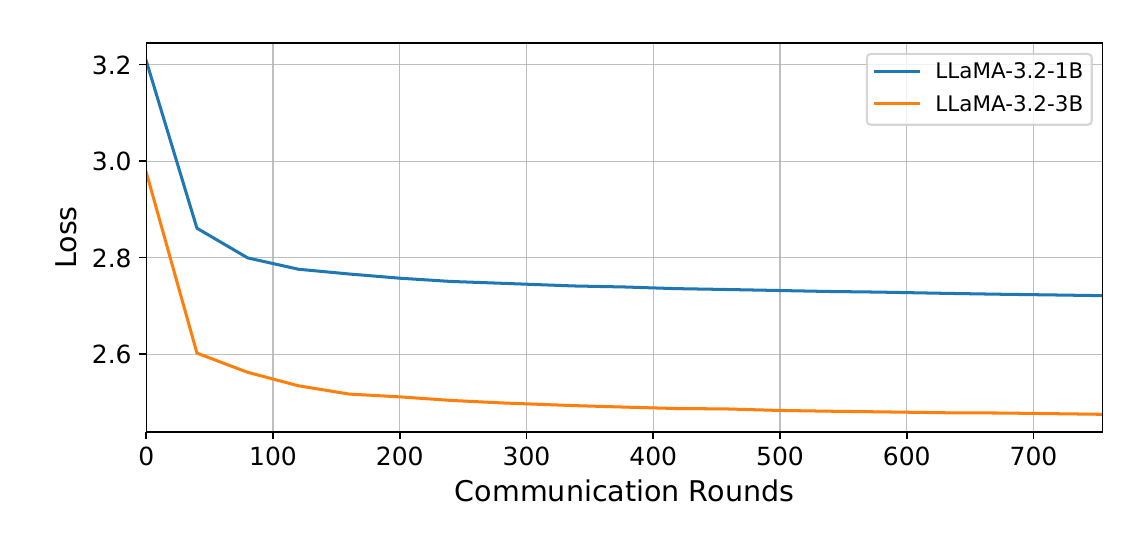}
          \caption{Test}
    \end{subfigure}
    \caption{NLP fine-tuning: FedSub-CD performance on training set (a) and test set (b) across model scales.}
    \label{fig:model-scaling}
\end{figure*}

Moreover, to evaluate the scalability of FedSub-CD with respect to model size, we fine-tune the LLaMA-3.2-1B and LLaMA-3.2-3B pre-trained models with their default tokenizers on the WikiText-2-Raw-v1 dataset. The detailed hyperparameter settings are provided in Table~\ref{tab:hparams} of the Supplementary Material.

As shown in Fig.~\ref{fig:model-scaling}, for both training and test losses, FedSub-CD consistently achieves a stable decrease as the model scale increases from 1B to 3B. Specifically, the test loss of the 3B model decreases from 2.97 to 2.47, while that of the 1B model decreases from 3.20 to 2.72. These results demonstrate that FedSub-CD can maintain and further improve fine-tuning performance as the model scale increases.

\begin{figure}
    \centering
    \includegraphics[width=1\linewidth]{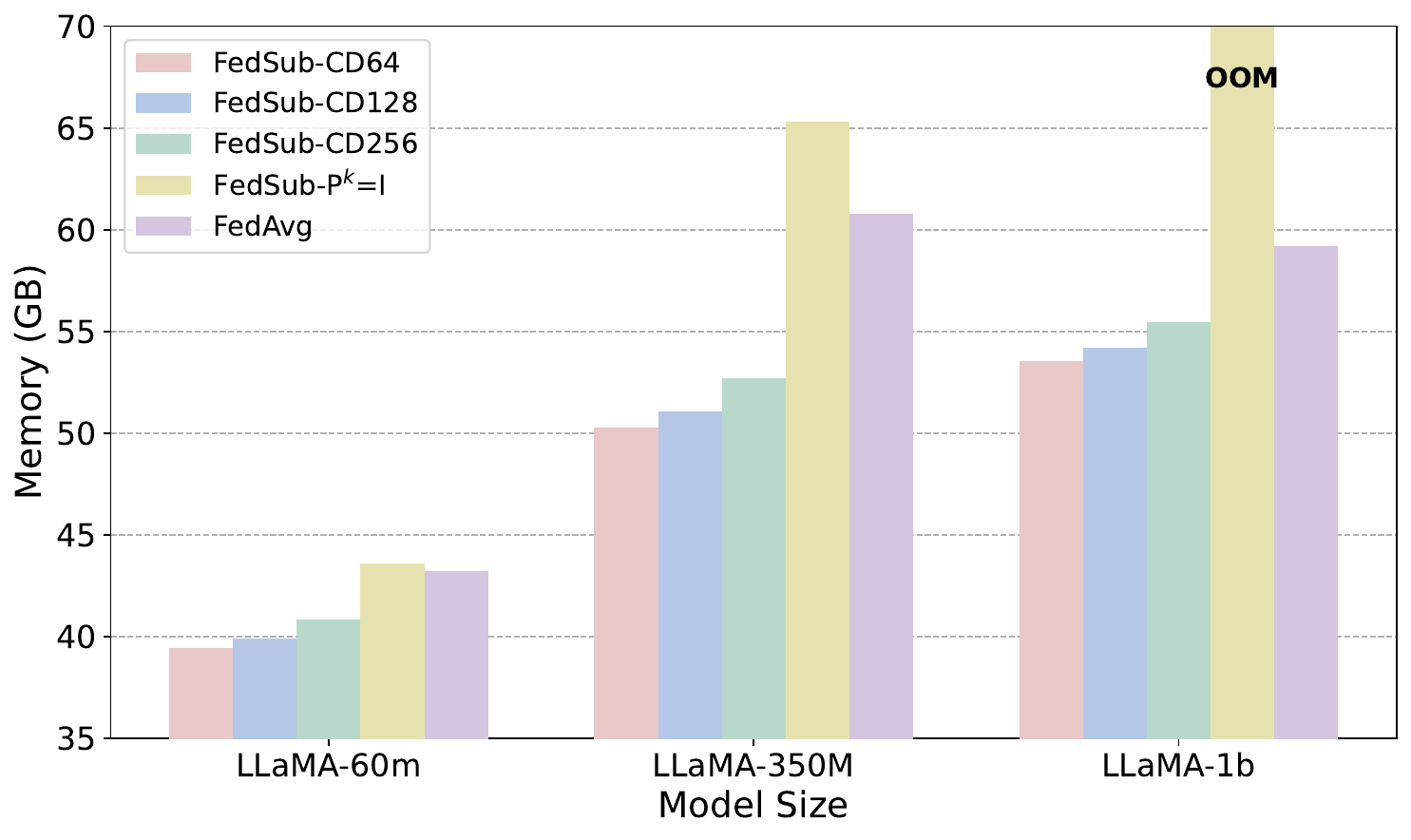}
    \caption{Memory efficiency for LLMs pre-training: To show memory efficiency compared to FedAvg, we add the corresponding percentage above each bar in the chart. In the LLaMA-1B test, the {FedSub-}$P^k = \mathrm{I}$ method runs out of memory, so it is not shown in the figure. For {FedSub-}CD, the value of $r$ is marked in the legend. For FedSub-$P^k = \mathrm{I}$, ``OOM" means ``out of memory". We use identical mixed-precision and gradient checkpointing settings across all methods for a fair comparison.}
    \label{fig:mem}
\end{figure}

\subsubsection{Memory Efficiency}
To measure the memory efficiency of our FedSub, we evaluate memory usage during pre-training of LLMs at different model scales, including LLaMA-60M, LLaMA-350M, and LLaMA-1B, and compare the results with the baseline FedAvg.

As shown in Fig.~\ref{fig:mem},  {FedSub-}CD achieves a favorable balance between model performance and memory efficiency: while maintaining comparable performance to FedAvg (as detailed in Fig.~3), it reduces memory usage by approximately 10\%. {FedSub-}$P^k = \mathrm{I}$ requires storing dual variables during training, resulting in slightly higher memory consumption compared to FedAvg. In contrast, {FedSub-}CD, as an improved variant of {FedSub-}$P^k = \mathrm{I}$, introduces projection matrices, which effectively reduce memory usage.

\begin{table}[htbp]
\centering
\caption{Average communication volume per training step (i.e., the total communication volume divided by the total number of training steps $K\tau$) for LLaMA-1B fine-tuning.}
\label{tab:commvolumn}
\begin{tabular}{lc}
\toprule
\textbf{Algorithm} & \textbf{Communication Volume (MB)}\\
\midrule
\textbf{{FedSub-}CD64}  & 9.673  \\
\textbf{{FedSub-}CD128}  & 19.345  \\
\textbf{{FedSub-}CD256}  & 38.690  \\
{FedSub-}$P^k = \mathrm{I}$  & 392.028  \\
FedAvg  & 494.326  \\
\bottomrule
\end{tabular}
\end{table}

\begin{figure*}[htbp]
    \centering
    \includegraphics[width=\textwidth]{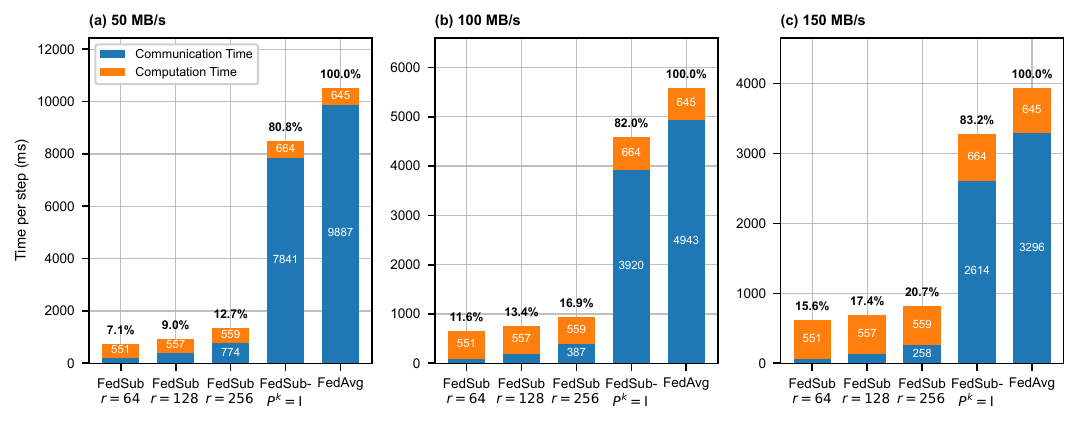}
    \caption{Runtime per step for LLaMA-1B fine-tuning under different communication bandwidths.}
    \label{fig:running time}
\end{figure*}

\subsubsection{Communication Efficiency}

We measure the average communication volume per training step (i.e., the total communication volume divided by the total number of training steps $K\tau$) in LLaMA-1B fine-tuning.

As shown in Table~\ref{tab:commvolumn}, our FedSub greatly reduces communication overhead. For example, when \(r = 64\), the average communication volume per training step of our FedSub is only approximately 2.0\% of that of FedAvg. This reduction comes from FedSub only sending low-dimensional subspace variables. The communication cost of {FedSub-}$P^k = \mathrm{I}$ is lower than that of FedAvg because {FedSub-}$P^k = \mathrm{I}$ only updates the attention and FFN parameters, while the remaining parameters are updated using SGD. In contrast, FedAvg updates all parameters in every step.

\subsubsection{Runtime Performance}

To validate the training speed of our algorithms, we evaluate the runtime, decomposing it into communication and computation components. For the LLaMA-1B fine-tuning experiments, we further conduct a bandwidth sensitivity analysis under communication bandwidths of 50, 100, and 150 MB/s. The communication time is computed from the communication volume and bandwidth, while the computation time is obtained by subtracting the communication time from the measured total runtime. The total runtime is the sum of the communication and computation times. The average communication time per step is calculated by dividing the total communication time by the total number of steps \(K\tau\), and the average computation time per step is computed analogously.

As shown in Fig.~\ref{fig:running time}, all three FedSub-CD configurations with \(r=64\), \(128\), and \(256\), as well as FedSub-\(P^k = \mathrm{I}\), achieve lower total runtime than FedAvg. In particular, FedSub-CD with \(r=64\) requires only 7.1\%, 11.6\%, and 15.6\% of the total runtime of FedAvg under communication bandwidths of 50, 100, and 150 MB/s, respectively. The runtime advantage becomes more pronounced as the communication bandwidth decreases, demonstrating the effectiveness of reducing communication overhead through subspace optimization.

\begin{figure*}[htbp]
    \centering
    \begin{subfigure}[b]{0.48\textwidth}
        \centering
        \includegraphics[width=\textwidth]{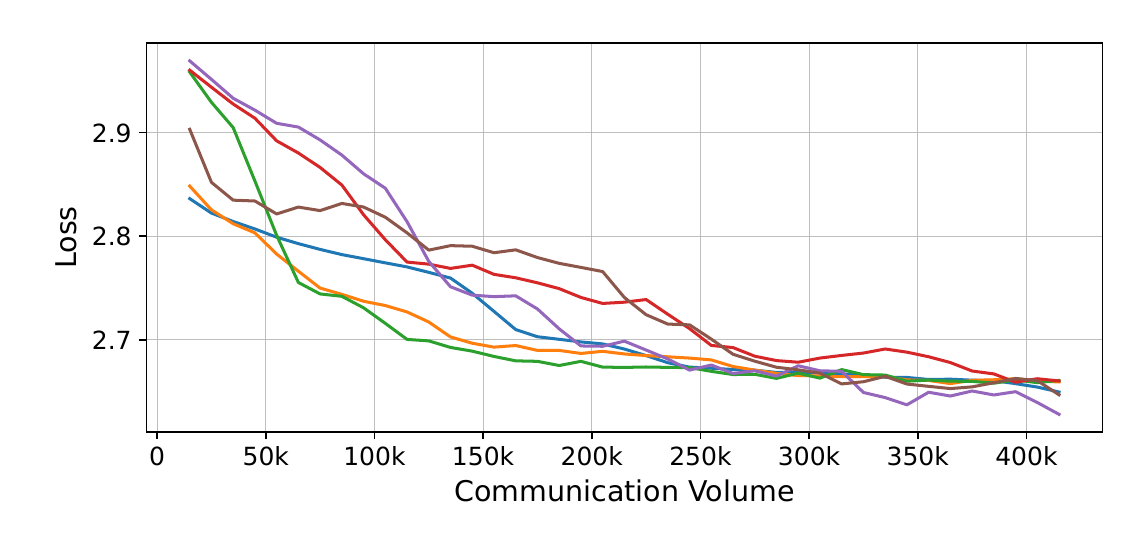}
          \caption{Training}
    \end{subfigure}
    \hfill
    \begin{subfigure}[b]{0.48\textwidth}
        \centering
        \includegraphics[width=\textwidth]{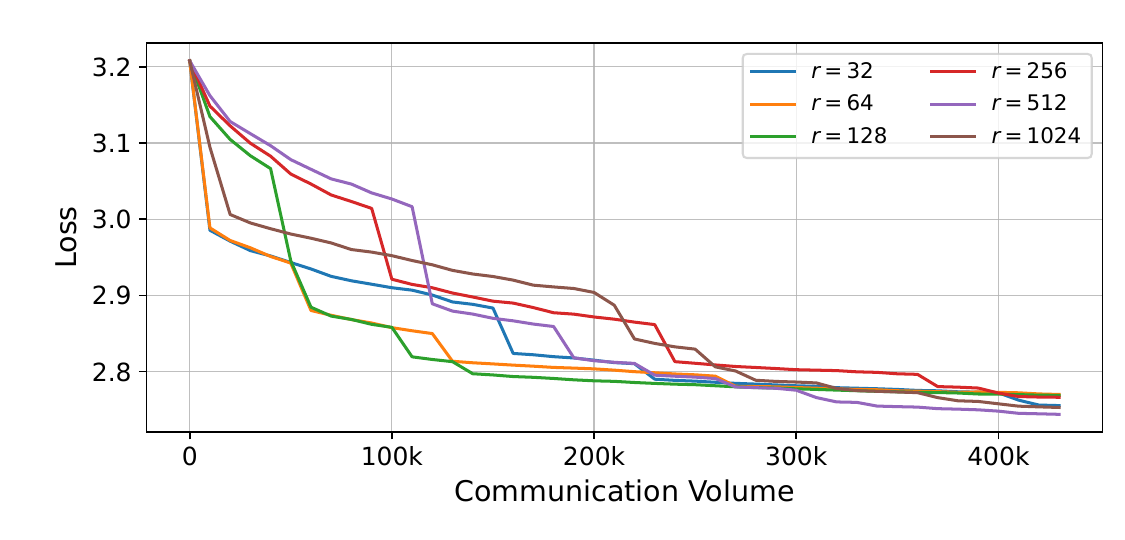}
          \caption{Test}
    \end{subfigure}
    \caption{NLP fine-tuning loss in LLaMA-3.2-1B on training set (a) and test set (b) under different $r$.}
    \label{fig:fixed-dimension}
\end{figure*}

\subsubsection{Impact of $r$}
We further investigate the effect of the subspace dimension \(r\) on LLM fine-tuning. We fine-tune the publicly available LLaMA-3.2-1B model on the WikiText-2-Raw-v1 dataset using its default tokenizer. 

As shown in Fig.~\ref{fig:fixed-dimension}, among the tested subspace dimensions, \(r=512\) achieves the lowest training and validation losses using the same communication volume.
The results indicate that an appropriately chosen $r$ can effectively balance accuracy and communication efficiency. While a larger \(r\) reduces error due to subspace, it also increases communication overhead.

\section{Conclusions}

We propose FedSub, a subspace-based algorithm for large-scale FL under data heterogeneity. By restricting the local updates of each client to low-dimensional subspaces, FedSub substantially reduces communication, computation, and memory overhead. In addition, it employs dual variables to effectively mitigate client drift, thereby enhancing convergence stability and overall performance. Theoretical analysis and experimental results further validate its efficiency. Several directions remain worthy of further exploration, including extending our FedSub to adaptive optimizers such as Adam. Moreover, future work could incorporate error compensation techniques to alleviate the bias caused by subspace projections. 

\makeatletter
\let\main@bibitem\bibitem
\renewcommand{\bibitem}[2][]{%
  \if\relax\detokenize{#1}\relax
    \main@bibitem{#2}%
  \else
    \main@bibitem[#1]{#2}%
  \fi
  \phantomsection\label{bib:#2}}
\makeatother
\bibliography{refs}
\bibliographystyle{IEEEtran}

\section{Appendix}\label{sec-app}

\subsection{Dual Variables to Mitigate Client Drift} \label{sec-simplify}

In this section, we show how the dual variables mitigate client drift by reformulating Algorithm~\ref{alg-fl} into an equivalent form, as follows.

In \eqref{our-derive}, for the local updates, we have
\begin{equation}\label{eq-b20}
\begin{aligned}
&B_i^{k,t+1} =  B_i^{k,t} -\eta \left({g_i(B_i^{k,t})}+\frac{1}{\eta \tau} \Lambda_i^k \right), \\
&B_i^{k,\tau}  =  -\eta \left( \sum_{t=0}^{\tau-1} {g_i(B_i^{k,t})}+\frac{1}{\eta} \Lambda_i^k \right),
\end{aligned}    
\end{equation}
where we define ${g_i(B_i^{k,t})}:={\frac{r}{m}} (P^k)^T\hat{\nabla} f_i(x^k+P^k B_i^{k,t})$ for simplicity and use  $B_i^{k,0}=0$ in the second update.

In \eqref{our-derive}, we recall that the dual variable is given by
\begin{equation}\label{eq-b21}
\Lambda_i^{k+1}=(P^{k+1})^T P^k\left(\Lambda_i^k + B_i^{k,\tau} -\frac{1}{n}\sum_{i=1}^n B_i^{k,\tau}\right). 
\end{equation}
Given the initialization $\Lambda_i^0 = 0$ and the fact that $P^k$ is identical across all clients, \eqref{eq-b21} implies that
\begin{equation}\label{eq-b22}
 \frac{1}{n} \sum_{i=1}^n \Lambda_i^k =0, ~\forall k.   
\end{equation}
By combining \eqref{eq-b22} and \eqref{eq-b20}, \eqref{eq-b21} can be written as
\begin{equation}\label{eq-b23}
\begin{aligned}
\Lambda_i^{k+1}&=(P^{k+1})^T P^k\left(\Lambda_i^k-\eta  \left(\sum_{t=0}^{\tau-1}{g_i(B_i^{k,t})} +\frac{1}{\eta}\Lambda_i^k\right) \right.\\
&\left.  + \eta\frac{1}{n}\sum_{i=1}^n \left(\sum_{t=0}^{\tau-1}{g_i(B_i^{k,t})} +\frac{1}{\eta}\Lambda_i^k\right)\right) \\
&= (P^{k+1})^T P^k\eta \sum_{t=0}^{\tau-1} \left( \bar{g}(B_i^{k,t}) - g_i(B_i^{k,t})\right),
\end{aligned}    
\end{equation}
where 
$\bar{g}(B_i^{k,t}) := \frac{1}{n} \sum_{i=1}^n g_i(B_i^{k,t}).
$
This shows that the dual variable $\Lambda_i^{k+1}$ essentially captures the difference between the average gradient across all clients and the local gradient.

Substituting \eqref{eq-b23} into \eqref{eq-b20} gives
\begin{align}\label{eq-def-c}
B_i^{k,t+1}&= B_i^{k,t} -\eta \Big({g_i(B_i^{k,t})}\\ \notag
& +\underbrace{ {(P^k)^T P^{k-1}} \frac{1}{\tau}\sum_{t=0}^{\tau-1} \left(\bar{g}(B_i^{k-1,t}) - g_i(B_i^{k-1,t}) \right)}_{:=c_i^k=\frac{1}{\eta\tau}\Lambda_i^k}\Big).  
\end{align}

In \eqref{our-derive}, the global model on the server is updated as
\begin{equation}\label{eq-g-16}
\begin{aligned}
x^{k+1}= x^k+P^k \frac{1}{n}\sum_{i=1}^n B_i^{k,\tau}= x^k-\eta P^k \sum_{t=0}^{\tau-1} {\bar{g}(B_i^{k,t})},
\end{aligned}    
\end{equation}
where the second equality follows from \eqref{eq-b20} and \eqref{eq-b22}.

In summary, combining \eqref{eq-b20} and \eqref{eq-g-16} to eliminate the dual variable, Algorithm~\ref{alg-fl} is expressed equivalently as
\begin{equation}\label{eq-28}
\begin{aligned}
B_i^{k,t+1}&= B_i^{k,t} -\eta \left({g_i(B_i^{k,t})}+c_i^k\right),\\
x^{k+1}&=x^k-\eta P^k \sum_{t=0}^{\tau-1} {\bar{g}(B_i^{k,t})}. 
\end{aligned}    
\end{equation}
The term $c_i^k$ is a correction for client drift, and from \eqref{eq-b23} we know that $c_i^k$ is in fact the scaled dual variable $\frac{1}{\eta\tau }\Lambda_i^k$.

For convenience of analysis, we will analyze \eqref{eq-28} to establish the convergence of Algorithm~\ref{alg-fl}.

\subsection{Efficiency of Algorithm \ref{alg-fl}} \label{app-eff-analysis}

The incorporation of subspace methods reduces the communication, computation, and memory complexity of Algorithm~\ref{alg-fl}, as summarized in Table~\ref{tab}. For simplicity, the following discussion focuses on a single layer.

\paragraph{Communication Efficiency} In Line 10 of Algorithm \ref{alg-fl},  client \( i \) sends the low-dimensional subspace variable \( B_{i}^{k,\tau} \in \mathbb{R}^{r \times d} \) to the server. Compared to the full-space setting, where each client would transmit a local update in \( \mathbb{R}^{m \times d} \), this reduces the uplink communication cost from \( \mathcal{O}\left(  md \right) \) to \( \mathcal{O}\left(  r d \right) \).

In Line 12, the server broadcasts the full-space global model \( x^{k+1} \in \mathbb{R}^{m \times d} \) to all clients. Since server bandwidth is typically sufficient, this does not cause a communication bottleneck in practice.


\paragraph{Computation Efficiency}  
In terms of computation complexity, we consider only gradient evaluations and matrix multiplications, while ignoring additions and scalar multiplications, because the latter are typically much less computationally expensive in practice. 
In Line 8, the product \( P^k B_i^{k,t} \) costs \( mrd \). The computation cost of computing the subspace gradient $(P^k)^T \hat{\nabla} f_i$ ({without explicitly computing $\hat{\nabla} f_i$}) is denoted as $C_{g}(rd)$. In Line 14, matrix multiplication \( P^k (\cdot) \) costs \( mrd \). Multiplying the result by \( (P^{k+1})^T \)  costs \( mrd \).  
Therefore, the total computation complexity of our FedSub is $ \mathcal{O}( \tau( mrd+C_{g}(rd) )+ 2mrd) $. 

When $P^k = \mathrm{I}$, the total computation cost reduces to that of evaluating the full-space gradient for $\tau$ times, which is $\mathcal{O}(\tau C_g(md))$. Similarly, for FedAvg, the total computation complexity is $\mathcal{O}(\tau C_g(md))$.

\paragraph{Memory Efficiency} We consider the memory cost per communication round of Algorithm \ref{alg-fl}, and compare with the full-space version when $P^k = \mathrm{I}$, and FedAvg that does not use subspace and dual variable.

In Line 8, due to the subspace projection, it suffices to compute the gradient of the mini-batch loss, defined as $\hat{f}_i(x^k + P^k B_i)$, with respect to the variable $B_i$. In other words, we can directly compute the overall quantity $(P^k)^T \hat{\nabla} f_i(x^k + P^k B_i^{k,t})$ without explicitly evaluating $\hat{\nabla} f_i(x^k + P^k B_i^{k,t})$ itself.  
Thus, in Line 8, our FedSub needs to store 
$$B_i^{k,t}: rd,~ P^k: rm,~ (P^k)^T\hat{\nabla} f_i: M_g(rd),  ~x^k: md, ~\Lambda_i^k: rd.  $$
Here, $M_g(rd)$ denotes the memory cost of computing a gradient of dimension $r\times d$, which is typically much higher than that of storing an $r \times d$ matrix in deep learning, due to additional memory required for intermediate quantities such as activations during backpropagation.

In Line 14, our FedSub needs to further store 
$$P^{k+1}: rm,~ \frac{1}{n}\sum_{i=1}^n B_i^{k,\tau}: rd.$$
The total memory cost for our FedSub is $\mathcal{O}( 3rd+M_g(rd)+2rm+md)$. 

By contrast, when $P^k = \mathrm{I}$, we can update the combined variable $B_i^{k,t} + x^k$ directly, without storing $B_i^{k,t}$ and $x^k$ separately. Thus, our FedSub with $P^k = \mathrm{I}$ needs to store 
$$B_i^{k,t} + x^k: md, ~\hat{\nabla} f_i: M_g(md),  ~\Lambda_i^k: md, ~\frac{1}{n}\sum_{i=1}^n B_i^{k,\tau}: md.$$
The total memory cost for our FedSub with $P^k = \mathrm{I}$ is $\mathcal{O}( 3md+M_g(md))$. Similarly, we can get that the total memory cost for FedAvg is $\mathcal{O}(md+M_g(md))$. 

As noted in Remark \ref{remark-scaffold}, Scaffold \cite{karimireddy2020scaffold} is mathematically equivalent to FedSub-$P^k = \mathrm{I}$, but requires clients to send additional correction terms. Its computation and memory complexity match FedSub-$P^k = \mathrm{I}$, while communication cost is doubled.

\subsection{Proof of Theorem \ref{thm}}\label{app-proof-thm}
The following inequality will be used repeatedly:
\begin{align}\label{eq-18-1}
& \|P^T\nabla f(x+PB)\|^2\\ \notag
&=
\sum_{l=1}^L
\|P_l^T\nabla_{x_l}f(x+PB)\|_F^2\\\notag
&\le
\sum_{l=1}^L
\lambda_{\max}(P_lP_l^T)\,
\|
\nabla_{x_l}f(x+PB)
\|_F^2\\\notag
&=
\sum_{l=1}^L
\lambda_{\max}(P_l^TP_l)\,
\|
\nabla_{x_l}f(x+PB)
\|_F^2\\\notag
&=
\sum_{l=1}^L
\frac{m_l}{r_l}
\|
\nabla_{x_l}f(x+PB)
\|_F^2\\\notag
&\le
\theta_m
\sum_{l=1}^L
\|
\nabla_{x_l}f(x+PB)
\|_F^2\\\notag
&=
\theta_m
\|
\nabla f(x+PB)
\|^2.
\end{align}
Here, the first inequality follows from 
$ \|AX\|_F^2
\le \lambda_{\max}(A^TA)\|X\|_F^2,$ 
which holds for arbitrary matrices $A$ and $X$. The third equality
follows directly from Assumption~\ref{asm-p}. We emphasize that \eqref{eq-18-1} holds pathwise for every realization of $P$, even in the presence of statistical dependence between $P$ and $\nabla f(x+PB)$. 

\begin{lemma}\label{lem-B}
Under Assumptions~\ref{asm-p}--\ref{asm-sgd-var}, and with the step size $\eta$ selected according to \eqref{eq-summarize-eta}, we obtain
\begin{equation}\label{eq-b27}
\begin{aligned}
&\sum_{i=1}^n\sum_{t=0}^{\tau-1} \bE\|B_i^{k,t}\|^2\le  5 \tau \sum_{i=1}^n \bE\| Y_i^k -\overline{Y}^k \|^2\\ & +5 \eta^2 {\theta_r^2} \tau^3 n \bE\|(P^k)^T\nabla f(x^k) \|^2+10\eta^2{{\theta_r^2} \theta_m\tau^2 n \frac{\sigma^2}{b}},
\end{aligned}    
\end{equation} 
where $\overline{Y}^k=\frac{1}{n}\sum_{i=1}^n Y_i^k$ and 
\begin{equation}\label{eq-def-Y}
\begin{aligned}
Y_i^k &:= \eta \tau\Big( {{\frac{r}{m}}(P^k)^T\nabla f_i(x^k)} \\ 
& +(P^k)^T P^{k-1} \frac{1}{\tau}\sum_{t=0}^{\tau-1} (\bar{g}(B_i^{k-1,t}) -g_i(B_i^{k-1,t}) ) \Big).    
\end{aligned}    
\end{equation}

\end{lemma}

\begin{proof}See Section~\ref{supp-lem1} of the Supplementary Material. \end{proof}

\begin{lemma}\label{lem-Y}
Under Assumptions \ref{asm-p}--\ref{asm-Bg}, if the step size $\eta$ satisfies \eqref{eq-summarize-eta}, we have
\begin{equation*}
\begin{aligned}
&\frac{1}{n}\sum_{i=1}^n \bE \| Y_i^{k+1} - \overline{Y}^{k+1} \|^2\\
&\le 2 \eta^2 {\theta_r^2}\tau^2  L_f^2  \bE\|x^{k+1}-x^k \|^2 + {5 \eta^2{\theta_r^2} \tau^2  \theta_m^2 \frac{1}{\tau}\frac{\sigma^2}{b} }\\
&+  8\eta^2{\theta_r^2}\tau^2  (\theta_m-1)B_g^2+4\eta^2{\theta_r^2}\tau^2  2 {\theta_m^3} L_f^2  \frac{1}{n\tau}\times\\
&\left(5 \tau \sum_{i=1}^n \bE\| Y_i^k -\overline{Y}^k \|^2 +5 \eta^2 {\theta_r^2} \tau^3 n \bE\|(P^k)^T\nabla f(x^k) \|^2\right).
\end{aligned}    
\end{equation*}    
\end{lemma}
\begin{proof}See Section~\ref{supp-lem2} of the Supplementary Material.  \end{proof}

\begin{lemma}\label{lem-lyp}
Under Assumptions \ref{asm-p}--\ref{asm-Bg}, if the step size $\eta$ satisfies \eqref{eq-summarize-eta}, we have 
\begin{equation*}
\begin{aligned}
\bE [\Omega^{k+1}] 
&\!\le\! \bE [\Omega^k]  \!-\! \frac{\eta \tau}{8\theta_m}  \bE \| \nabla f(x^k)\|^2 + 4 \eta^2{\theta_r^2}\tau^2 \theta_m \! \left(\theta_m\!-\!1 \right)\! {\frac{B_g^2}{n}}  \\
&+\eta^2{\theta_r^2} \tau^2 \theta_m^2 \frac{1}{n\tau} \frac{\sigma^2}{b}  \left( 2L_f+{\theta_r^2}4L_f^2 \eta^2 \tau^2 {\frac{1}{n}}+ 5\right).  
\end{aligned}
\end{equation*}
\end{lemma}
\begin{proof}
Let 
\[
\overline{\hat{\nabla} f}(x^k + P^k B_i^{k,t}) := \frac{1}{n} \sum_{i=1}^n \hat{\nabla} f_i(x^k + P^k B_i^{k,t}),
\] 
and similarly, let $\overline{\nabla f}$ denote the average of the full gradients $\nabla f_i$.  
By the $L_f$-smoothness of $f$, it follows that
\begin{align}\label{eq-49}
&\bE [f(x^{k+1})] \\ \notag
&\le \bE [f(x^k) + \langle \nabla f(x^k), x^{k+1}-x^k \rangle + \frac{L_f}{2} \|x^{k+1}-x^k \|^2]\\\notag
&=\! \bE [f(x^k) \!-\! \eta\tau \big\langle \nabla f(x^k),\!  \frac{r}{m}P^k   (P^k)^T \frac{1}{\tau}\sum_{t=0}^{\tau-1}\! \overline{ \nabla f}(x^k\!+\!P^k B_i^{k,t}) \big\rangle] \\\notag
&+ {\theta_r^2}\frac{\eta^2\tau^2 L_f}{2} \bE \| P^k   (P^k)^T\frac{1}{\tau}\sum_{t=0}^{\tau-1} \overline{\hat{\nabla} f}(x^k+P^k B_i^{k,t})\|^2 \\\notag
&\le \bE [f(x^k)] -\frac{\eta\tau}{2 {\theta_m}} \bE\|{(P^k)^T} \nabla f(x^k)\|^2 \\\notag
&-\frac{\eta\tau}{2{\theta_m}} \bE\| (P^k)^T \frac{1}{\tau}\sum_{t=0}^{\tau-1} \overline{ \nabla f}(x^k+P^k B_i^{k,t}) \|^2 \\\notag
&+ {\theta_r}\frac{\eta\tau}{2} \underbrace{\bE\| (P^k)^T \frac{1}{\tau}\sum_{t=0}^{\tau-1} \left(\overline{ \nabla f}(x^k+P^k B_i^{k,t}) - \nabla f(x^k) \right)\|^2}_{\rm (IV.1)}\\\notag
&+{\theta_r^2}\frac{\eta^2\tau^2 L_f}{2} \theta_m \underbrace{\bE\|   (P^k)^T\frac{1}{\tau}\sum_{t=0}^{\tau-1} \overline{\hat{\nabla} f}(x^k+P^k B_i^{k,t})\|^2}_{\rm (IV.2)},\notag
\end{align}
where we substitute $$ x^{k+1}
=x^k-\eta\tau \frac{r}{m} P^k(P^k)^T \frac{1}{\tau}\sum_{t=0}^{\tau-1}   \overline{\hat{\nabla} f}(x^k+P^k B_i^{k,t})$$ in the first inequality,  use the fact that the randomness of \( P^k \) and the randomness of the stochastic gradients are independent in the first equality, and use $\min_l \frac{r_l}{m_l}=(\max_l \frac{m_l}{r_l})^{-1}$ in the last inequality.

For the term (IV.1), we have 
\begin{equation*}
\begin{aligned}
&{\rm (IV.1)}\le {\theta_m} \bE \|  \frac{1}{\tau}\sum_{t=0}^{\tau-1} \left( \overline{ \nabla f}(x^k+P^k B_i^{k,t}) -  \nabla f(x^k) \right)\|^2 \\
&\le {\theta_m} \frac{1}{\tau} \sum_{t=0}^{\tau-1} \bE\|  \frac{1}{n}\sum_{i=1}^n \left(\nabla f_i(x^k+P^k B_i^{k,t}) -  \nabla f_i(x^k)\right) \|^2\\
&\le {\theta_m}L_f^2 \frac{1}{\tau n} \sum_{t=0}^{\tau-1} \sum_{i=1}^n  \bE\|   P^k B_i^{k,t} \|^2.
\end{aligned}    
\end{equation*}

For the term (IV.2), we have 
\begin{align*}
&{\rm (IV.2)}=\\
&\bE \|   (P^k)^T\frac{1}{\tau}\sum_{t=0}^{\tau-1} \left( \overline{\hat{\nabla} f}(x^k+P^k B_i^{k,t}) \pm\overline{{ \nabla f}}(x^k+P^k B_i^{k,t}) \right)\|^2 \\
&\le 2\bE\|(P^k)^T\frac{1}{\tau}\sum_{t=0}^{\tau-1} \overline{{ \nabla f}}(x^k+P^k B_i^{k,t})\|^2\\
&+\! 2\bE\| (P^k)^T\frac{1}{\tau}\sum_{t=0}^{\tau-1} \left(\overline{\hat{\nabla} f}(x^k\!+\!P^k B_i^{k,t}) \!-\!\overline{{ \nabla f}}(x^k\!+\!P^k B_i^{k,t})\right)\!\|^2\\
&\le 2\bE\|(P^k)^T\frac{1}{\tau}\sum_{t=0}^{\tau-1} \overline{{ \nabla f}}(x^k+P^k B_i^{k,t})\|^2+ 2\theta_m \frac{1}{n\tau} \frac{\sigma^2}{b},
\end{align*}    
where the notation $\pm a$ indicates the operation of adding and then subtracting $a$, and the last inequality follows from Assumption~\ref{asm-sgd-var}. 
Substituting (IV.1) and (IV.2) into \eqref{eq-49} then yields
\begin{align*}
&\bE [f(x^{k+1})] \le \bE [f(x^k)] -\frac{\eta\tau}{2 {\theta_m}} \bE\|{(P^k)^T} \nabla f(x^k)\|^2 \\
&-\!\frac{\eta\tau}{2{\theta_m}} \bE\| (P^k)^T \frac{1}{\tau}\sum_{t=0}^{\tau-1} \overline{ \nabla f}(x^k+P^k B_i^{k,t}) \|^2 \\
&+\! {\theta_r}\frac{\eta\tau}{2} {\theta_m}L_f^2 \frac{1}{\tau n} \sum_{t=0}^{\tau-1} \sum_{i=1}^n  \bE\|   P^k B_i^{k,t} \|^2 \\
&+\!{\theta_r^2}{\eta^2\tau^2 L_f} \theta_m \!\Bigg(\! \bE\|(P^k)^T\frac{1}{\tau}\sum_{t=0}^{\tau-1} \overline{{ \nabla f}}(x^k\!+\!P^k B_i^{k,t})\|^2+ \frac{\theta_m }{n\tau} \frac{\sigma^2}{b}\! \Bigg).
\end{align*}    
Rearranging the above inequality, we obtain
\begin{align}\label{eq-g-39}
&\bE [f(x^{k+1})\!+\!\frac{\eta\tau}{2\theta_m} {\left( 1\!-\!2\eta\tau L_f {\theta_r^2} \theta_m^2 \right)}\times \\\notag
& \|  (P^k)^T\frac{1}{\tau}\sum_{t=0}^{\tau-1} \overline{ \nabla f}(x^k\!+\!P^k B_i^{k,t}) \|^2]\\\notag
&\le\bE [f(x^k)] -\frac{\eta \tau }{2\theta_m}  \bE\| {(P^k)^T} \nabla f(x^k)\|^2+ {\theta_r}\frac{\eta }{2}   L_f^2{\theta_m^2} \frac{1}{n}\times\\\notag
&\left(5 \tau \sum_{i=1}^n  \bE\| Y_i^k -\overline{Y}^k \|^2 +5 \eta^2 {\theta_r^2} \tau^3 n \bE\|(P^k)^T\nabla f(x^k) \|^2\right)\\\notag
&+ {\eta^2\tau^2 L_f} {{\theta_r^2} \theta_m^2 \frac{1}{n\tau} \frac{\sigma^2}{b} + {\theta_r}\frac{\eta }{2}   L_f^2{\theta_m^2} \frac{1}{n} 10\eta^2 {\theta_r^2}  \theta_m \tau^2 n \frac{\sigma^2}{b}  } \\\notag
&\le\bE [f(x^k)] -\frac{\eta \tau}{4\theta_m}  \bE\| {(P^k)^T} \nabla f(x^k)\|^2\\\notag
&+ {\theta_r}\frac{5\eta\tau }{2}   L_f^2{\theta_m^2}  \frac{1}{n}\sum_{i=1}^n \bE\| Y_i^k \!-\!\overline{Y}^k \|^2+2\eta^2\tau^2 L_f {{\theta_r^2}\theta_m^2 \frac{1}{n\tau} \frac{\sigma^2}{b}}, \notag
\end{align}    
where we use \eqref{eq-b27} in the first inequality and 
\begin{equation}\label{eq-stepsize-3}
\begin{aligned}
&{\eta^2\tau^2 L_f} {{\theta_r^2} \theta_m^2 \frac{1}{n\tau} \frac{\sigma^2}{b} \ge {\theta_r}\frac{\eta }{2}   L_f^2{\theta_m^2} \frac{1}{n} 10\eta^2 {\theta_r^2}  \theta_m \tau^2n\frac{\sigma^2}{b}  },  \\
&{\theta_r}\frac{\eta }{2}   L_f^2{\theta_m^2} \frac{1}{n}\left(5 \eta^2 {\theta_r^2} \tau^3 n \right) \le \frac{\eta\tau}{4\theta_m} 
\end{aligned}    
\end{equation}
in the last inequality.
Recalling Lemma~\ref{lem-Y}, we introduce the following definition:
\begin{equation}
\Omega^{k+1}: = f(x^{k+1})-f^{\star} +\frac{1}{n^2}\sum_{i=1}^n\| Y_i^{k+1} - \overline{Y}^{k+1} \|^2,  
\end{equation}
where $f^{\star}$ denotes the optimal value of problem~\eqref{eqn:basic_opt}.  
It then follows from \eqref{eq-g-39} that
\begin{align}\label{eq-b52}
&\bE [\Omega^{k+1}] \le \bE [f(x^k)]-f^{\star}-\frac{\eta \tau}{4\theta_m}  \bE \| {(P^k)^T} \nabla f(x^k)\|^2\\ \notag
&+ {\theta_r}\frac{5\eta\tau }{2}   L_f^2{\theta_m^2}  \frac{1}{n}\sum_{i=1}^n \bE\| Y_i^k -\overline{Y}^k \|^2 +2\eta^2\tau^2 L_f {{\theta_r^2}\theta_m^2 \frac{1}{n\tau} \frac{\sigma^2}{b}} \\\notag
&+ 2 \eta^2 {\theta_r^2}\tau^2  L_f^2  {\frac{1}{n}}\bE\|x^{k+1}-x^k \|^2 + 5 \eta^2{{\theta_r^2}\tau^2  \theta_m^2 {\frac{1}{n}}\frac{1}{\tau}\frac{\sigma^2}{b} }\\\notag
&+4\eta^2{\theta_r^2}\tau^2  2 {\theta_m^3} L_f^2  \frac{1}{n\tau}{\frac{1}{n}}\times\\\notag
&\left(5 \tau \sum_{i=1}^n \bE\| Y_i^k -\overline{Y}^k \|^2 +5 \eta^2 {\theta_r^2} \tau^3 n \bE\|(P^k)^T\nabla f(x^k) \|^2\right) \\\notag
&+  8 \eta^2{\theta_r^2}\tau^2  (\theta_m-1) {\frac{1}{n}}B_g^2 \\\notag
&-\!\frac{\eta\tau}{2\theta_m} {\left( 1\!-\!2\eta\tau L_f {\theta_r^2} \theta_m^2 \right)}{ \bE \| (P^k)^T \frac{1}{\tau}\sum_{t=0}^{\tau-1} \overline{ \nabla f}(x^k\!+\!P^k B_i^{k,t}) \|^2}.\notag
\end{align}    
To bound $\mathbb{E}\|x^{k+1} - x^k\|^2$ in \eqref{eq-b52}, we substitute
$x^{k+1}= x^k-\eta \frac{r}{m} P^k  (P^k)^T\sum_{t=0}^{\tau-1}  \overline{\hat{\nabla} f}(x^k+P^k B_i^{k,t})$ to get 
\begin{align}\label{eq-z-48}
&\bE \|x^{k+1}- x^k\|^2\\ \notag
&\le \eta^2{\theta_r^2} \theta_m \bE \| (P^k)^T\sum_{t=0}^{\tau-1}  \overline{\hat{\nabla} f}(x^k+P^k B_i^{k,t}) \|^2\\\notag
&\le  2\eta^2{\theta_r^2} \theta_m \bE \| (P^k)^T\sum_{t=0}^{\tau-1}  \overline{ \nabla f}(x^k+P^k B_i^{k,t}) \|^2 \\\notag
&+\! 2\eta^2{\theta_r^2} \theta_m \bE \| (\!P^k\!)^T\!\sum_{t=0}^{\tau-1}  (\!\overline{\hat{\nabla} \!f}(x^k\!+\!P^k B_i^{k,t}) \!-\!\overline{ \nabla \!f}(x^k\!+\!P^k B_i^{k,t}) \!) \|^2\\\notag
&\le  2\eta^2{\theta_r^2} \theta_m \bE  \| (P^k)^T\sum_{t=0}^{\tau-1}  \overline{ \nabla f}(x^k\!+\!P^k B_i^{k,t}) \|^2 + {2\eta^2{\theta_r^2} \theta_m^2  \frac{\tau}{n} \frac{\sigma^2}{b}},
\end{align}    
where the last inequality follows from Assumption~\ref{asm-sgd-var}. By substituting \eqref{eq-z-48} into \eqref{eq-b52}, we obtain
\begin{align}\label{eq-g-44}
&\bE[ \Omega^{k+1}]\\ \notag
&\le \bE [f(x^k)\!-\!f^{\star}\!-\! \frac{\eta \tau}{8\theta_m}  \| {(P^k)^T} \nabla f(x^k)\|^2 + \frac{1}{n^2}\sum_{i=1}^n \| Y_i^k\! -\!\overline{Y}^k \|^2] \\\notag
&+  4 \eta^2{\theta_r^2}\tau^2 \theta_m  \left(\theta_m-1 \right) {\frac{1}{n}} B_g^2 \\\notag
&+\eta^2{\theta_r^2} \tau^2 \theta_m^2 \frac{1}{n\tau} \frac{\sigma^2}{b} 
\left( 2L_f+{\theta_r^2}4L_f^2 \eta^2 \tau^2 {\frac{1}{n}}+ 5\right), \notag
\end{align}    
where we substitute 
\begin{equation}\label{eq-stepsize-2}
\begin{aligned}
&1-2\eta\tau L_f {\theta_r^2} \theta_m^2 - \frac{2\theta_m}{\eta\tau}2 \eta^2 {\theta_r^2}\tau^2  L_f^2  {\frac{1}{n}} 2\eta^2{\theta_r^2} \theta_m \ge 0, \\
&\frac{\eta \tau}{4\theta_m}- {\theta_r^4} 4\eta^2 \tau^2 2L_f^2 \theta_m^3\frac{1}{n\tau} {\frac{1}{n}} 5\eta^2 \tau^3 n  \ge \frac{\eta\tau}{8\theta_m},\\
&{\theta_r}\frac{2 {n}5\eta\tau }{2}   L_f^2{\theta_m^2}\le1, \\
&4\eta^2 \theta_r^2\tau^2 2 \theta_m^2 L_f^2 \frac{1}{\tau}\frac{1}{n^2}5\tau \le \theta_r \frac{5\eta\tau}{2}L_f^2 \theta_m^2 \frac{1}{n}.
\end{aligned}
\end{equation}

Let $\{\mathcal F_k\}$ denote the algorithmic history immediately before
$P^k$ is sampled. Then, $x^k$ is measurable with respect to
$\mathcal F_k$, whereas $P^k$ is independent of $\mathcal F_k$. 
With $\mathbb{E}[P_l^k(P_l^k)^T]=I$ for all $l,k$, taking the conditional expectation with respect to $\mathcal F_k$ and then applying the tower property yields
\begin{align}\label{eq-g-46}
&\mathbb{E}\!\left[\|(P^k)^T\nabla f(x^k)\|^2\right] \notag\\
&=
\sum_{l=1}^L
\bE\!\left[
\mathbb E
\!\left[
\|(P_l^k)^T\nabla_{x_l}f(x^k)\|_F^2
\,\middle|\,
\mathcal F_k
\right]
\right]
\notag\\
&=
\sum_{l=1}^L
\bE\!\left[
\mathbb E
\!\left[
\operatorname{Tr}
\!\left(
(\nabla_{x_l}f(x^k))^T
P_l^k(P_l^k)^T
\nabla_{x_l}f(x^k)
\right)
\,\middle|\,
\mathcal F_k
\right]
\right]
\notag\\
&=
\sum_{l=1}^L
\bE\!\left[
\operatorname{Tr}
\!\left(
(\nabla_{x_l}f(x^k))^T
\mathbb E
\!\left[
P_l^k(P_l^k)^T
\,\middle|\,
\mathcal F_k
\right]
\nabla_{x_l}f(x^k)
\right)
\right]
\notag\\
&=
\sum_{l=1}^L
\bE\!\left[
\operatorname{Tr}
\!\left(
(\nabla_{x_l}f(x^k))^T
\nabla_{x_l}f(x^k)
\right)
\right]
\notag\\
&=
\bE\|\nabla f(x^k)\|^2.
\end{align}
Here, $\operatorname{Tr}(\cdot)$ denotes the trace of a matrix. The fourth equality follows from the independence of $P_l^k$ and $\mathcal F_k$, together with
$ \mathbb E[P_l^k(P_l^k)^T]=I. $

Then, substituting \eqref{eq-g-46} into \eqref{eq-g-44}, we have 
\begin{equation}\label{eq-g-48}
\begin{aligned}
 &\bE [\Omega^{k+1}] 
\le \bE [\Omega^k] \\
&- \frac{\eta \tau}{8\theta_m}  \bE \| \nabla f(x^k)\|^2 + 4 \eta^2{\theta_r^2}\tau^2 \theta_m  \left(\theta_m-1 \right) {\frac{1}{n}} B_g^2 \\
&+{\eta^2{\theta_r^2} \tau^2 \theta_m^2 \frac{1}{n\tau} \frac{\sigma^2}{b}  \left( 2L_f+{\theta_r^2}4L_f^2 \eta^2 \tau^2 {\frac{1}{n}}+ 5\right) },
\end{aligned}    
\end{equation}
which completes the proof of Lemma \ref{lem-lyp}. 
\end{proof}

In summary, all the step size conditions used above are satisfied by choosing
\begin{equation}\label{eq-summarize-eta}
\eta \le \mathcal{O}\left(\frac{1}{\theta_r\theta_m^2 L_f^2 n \tau}\right).  
\end{equation}
By recursively using \eqref{eq-g-48}, we get Theorem \ref{thm}.




\clearpage

\twocolumn[
\begin{@twocolumnfalse}
	\section*{\centering{Supplementary Material for \\ \emph{ An Efficient Subspace Algorithm for Federated Learning on Heterogeneous Data	\\[18pt]}}}
\end{@twocolumnfalse}
]
\setcounter{section}{7}
\setcounter{subsection}{3}
\setcounter{equation}{34}
\setcounter{table}{3}
\setcounter{figure}{9}
\setcounter{algorithm}{1}
\setcounter{theorem}{1}
\setcounter{lemma}{3}
\setcounter{assumption}{4}
\setcounter{remark}{3}

\begin{table*}[h]
\caption{Hyperparameter settings for ResNet training, pre-training, and fine-tuning: To further reduce the computation overhead caused by updating the subspace projection matrix $P^k$ and to enhance training stability, we adopt a periodic update strategy. The update interval is controlled by the parameter \textbf{Generation Freq}; for example, \textbf{Generation Freq} = 10 means that a new $P^k$ is randomly sampled every 10 communication rounds. 
For the subspace dimension $r$, two approaches are used: one is to set it as a fixed constant (e.g., $r = 3$); the other is to define it as a proportion of the original dimension. For instance, 70\% means setting $r$ to 70\% of the original dimension. 
}
\label{tab:hparams}
\centering
\resizebox{\textwidth}{!}{%
\begin{tabular}{lccccccc}
\toprule
\textbf{Tasks} & \textbf{Batch Size} & \textbf{Sequence Length} & \textbf{$r$} & \textbf{$\tau$} & \textbf{Learing Rate $\eta$} & \textbf{Generation Freq}  & \textbf{$P^k$}\\
\midrule
 ResNet training on {CIFAR-10/100} & 32 & - & 3 & 10 & 0.1 & 5  & CD \\
{ResNet (partial participation)} & {32} & {-} & {3} & {10} & {0.1} & {50}  & {CD}\\
Pre-training  & 64  & 1024 & 50\%  &  10 & 0.001 & 5 & SVD/CD \\
Fine-tuning & 16  & 512 & 70\% & 5 & 0.0001 & 10 &  CD\\
{Fine-tuning (runtime)} & {16} & {512} & {$\{64,128,256\}$} & {5} & {0.0001} & {10} &  {CD}\\
{Fine-tuning (model scaling)} & 16 & 512 & 70\% & 5 & 0.0001 & 10  & CD\\
{Fine-tuning (fixed-dimension ablation)} & {16} & {512} & {$\{16,32,64,128,256,512,1024\}$} & {5} & {0.0001} & {10} & {CD}\\
\bottomrule
\end{tabular}
}
\end{table*}

\subsection{Equivalence with Scaffold}\label{sec:scaffold-equivalence}
We first recall the Scaffold updates under the setting considered in this paper.
Specifically, we consider full client participation and the server stepsize
$\eta_g=1$.
The updates of Scaffold are
\begin{subequations}\label{eq:scaffold}
\begin{align}
y_i^{k,t+1}
&=
y_i^{k,t}
-
\eta
\left(
\hat{\nabla}f_i(y_i^{k,t})
+
e^k-e_i^k
\right), ~\forall t,
\label{eq:scaffold-a}
\\
e_i^{k+1}
&=
e_i^k-e^k
+
\frac1{\eta\tau}
\left(
x^k-y_i^{k,\tau}
\right),
\label{eq:scaffold-b}
\\
x^{k+1}
&=
x^k
+
\frac1n
\sum_{i=1}^n
\left(
y_i^{k,\tau}-x^k
\right),
\label{eq:scaffold-c}
\\
e^{k+1}
&=
e^k
+
\frac1n
\sum_{i=1}^n
\left(
e_i^{k+1}-e_i^k
\right),
\label{eq:scaffold-d}
\end{align}
\end{subequations}
where the initialization satisfies
\[
e^0=\frac1n\sum_{i=1}^n e_i^0,
\qquad
y_i^{k,0}=x^k.
\]
\begin{prop}\label{prop:scaffold}
Consider FedSub with $P^k=\mathrm{I}$ for all $k$ and
$\Lambda_i^0=0$ for all $i$. Consider Scaffold in
\eqref{eq:scaffold} with
$ e_i^0=e^0=0,\quad \forall i.$
If the two algorithms use the same realization of the stochastic gradients at
the corresponding local iterates, then they are algebraically equivalent and
generate identical local and global iterates.
\end{prop}
\begin{proof}
We compare the correction terms used by the two methods.  When
$P^k=\mathrm{I}$, we have $r=m$ and
$g_i(B_i^{k,t})=\hat{\nabla}f_i(x^k+B_i^{k,t})$.  Define the Scaffold local
iterate by
\[
y_i^{k,t}=x^k+B_i^{k,t}.
\]
Because $\Lambda_i^0=0$, the FedSub correction satisfies $c_i^0=0$.
Moreover, \eqref{eq-def-c} gives, for every $k\geq 1$,
\begin{equation}\label{eq:fedsub-c-full}
c_i^k
=
\frac1\tau\sum_{t=0}^{\tau-1}
\left(
\frac1n\sum_{j=1}^n
\hat{\nabla}f_j(y_j^{k-1, t})
-
\hat{\nabla}f_i(y_i^{k-1, t})
\right).
\end{equation}
Hence  the update \eqref{eq-28} of our Fedsub becomes
\begin{subequations}\label{eq:fedsub-full}
\begin{align}
y_i^{k,t+1}
&=
y_i^{k,t}
-\eta\left(
\hat{\nabla}f_i(y_i^{k,t})+c_i^k
\right),
\label{eq:fedsub-fulla}
\\
x^{k+1}
&=
x^k
-\eta\sum_{t=0}^{\tau-1}
\frac1n\sum_{i=1}^n
\hat{\nabla}f_i(y_i^{k,t}).
\label{eq:fedsub-fullb}
\end{align}
\end{subequations}

We next rewrite the Scaffold correction.  From
\eqref{eq:scaffold-d} and the initialization, induction gives
\begin{equation}\label{eq:scaffold-average}
e^k=\frac1n\sum_{i=1}^n e_i^k,\qquad \forall k.
\end{equation}
Summing \eqref{eq:scaffold-a} over the $\tau$ local iterations and using
$y_i^{k,0}=x^k$ yields
\begin{equation}\label{eq:scaffold-ytau}
y_i^{k,\tau}
=x^k-\eta\sum_{t=0}^{\tau-1}
\left(
\hat{\nabla}f_i(y_i^{k,t})+e^k-e_i^k
\right).
\end{equation}
Substituting \eqref{eq:scaffold-ytau} into
\eqref{eq:scaffold-b} gives
\begin{equation}\label{eq:scaffold-ei}
e_i^{k+1}
=
\frac1\tau\sum_{t=0}^{\tau-1}
\hat{\nabla}f_i(y_i^{k,t}).
\end{equation}
Averaging \eqref{eq:scaffold-ei} over the clients and using
\eqref{eq:scaffold-average} gives
\begin{equation}\label{eq:scaffold-e}
e^{k+1}
=
\frac1n\sum_{i=1}^n
\frac1\tau\sum_{t=0}^{\tau-1}
\hat{\nabla}f_i(y_i^{k,t}).
\end{equation}
Therefore, for every $k\geq1$,
\begin{equation*}
e^k-e_i^k
=
\frac1\tau\sum_{t=0}^{\tau-1}
\left(
\frac1n\sum_{j=1}^n
\hat{\nabla}f_j(y_j^{k-1, t})
-
\hat{\nabla}f_i(y_i^{k-1, t})
\right)
=c_i^k,
\end{equation*}
where the last equality follows from \eqref{eq:fedsub-c-full}.  For
$k=0$, the initialization gives
$e^0-e_i^0=0=c_i^0$.  Thus the Scaffold local update
\eqref{eq:scaffold-a} is exactly \eqref{eq:fedsub-fulla} in every round.

Finally, substituting \eqref{eq:scaffold-ytau} into
\eqref{eq:scaffold-c} and using \eqref{eq:scaffold-average} yields
\begin{align*}
x^{k+1}
&=x^k-\frac{\eta}{n}\sum_{i=1}^n
\sum_{t=0}^{\tau-1}
\left(
\hat{\nabla}f_i(y_i^{k,t})+e^k-e_i^k
\right)
\\
&=x^k-\eta\sum_{t=0}^{\tau-1}
\frac1n\sum_{i=1}^n
\hat{\nabla}f_i(y_i^{k,t}),
\end{align*}
which is exactly \eqref{eq:fedsub-fullb}.  Starting from the same $x^0$ and
using the same stochastic-gradient realization, induction on $k$ and $t$
now shows that FedSub and Scaffold generate identical local and global
iterates.
\end{proof}

\subsection{ Proof of Lemma \ref{lem-B}}\label{supp-lem1}
Repeated application of the first update in \eqref{eq-28} yields
\begin{equation}\label{eq-b28}
B_i^{k,t+1}=- \eta \sum_{s=0}^t \left( g_i(B_i^{k,s}) + c_i^k \right),
\end{equation}
where we use $B_i^{k,0} = 0$. By adding and subtracting ${\frac{r}{m}}\eta (t+1)(P^k)^T \nabla f(x^k)$ on the right-hand side of \eqref{eq-b28} and taking the norm of both sides, we obtain 
\begin{align}\label{eq-32}
&\|B_i^{k,t+1}\|^2\\ \notag
&\le 2 \underbrace{\| \eta \sum_{s=0}^t \left( g_i(B_i^{k,s}) + c_i^k \right) -{\frac{r}{m}}\eta(t+1)(P^k)^T\nabla f(x^k)\|^2}_{\rm (I)} \\\notag
&+2 {\|{\frac{r}{m}}\eta(t+1)(P^k)^T\nabla f(x^k) \|^2}.  
\end{align}  
For term (I), we rewrite it using $Y_i^k$, as defined in \eqref{eq-def-Y}, as
\begin{align}\label{eq-33}
&{\rm(I)}=\| \eta \sum_{s=0}^t \left( g_i(B_i^{k,s}) + c_i^k -{\frac{r}{m}} (P^k)^T \nabla f(x^k) \right)\|^2\\ \notag
&=\| \eta \!\sum_{s=0}^t \!\left( g_i(B_i^{k,s})\! \pm\! \frac{r}{m}(P^k)^T\nabla f_i(x^k) \!+\! c_i^k - \frac{1}{\eta\tau}\overline{Y}^k\right)\|^2\\\notag
&=\| \eta \sum_{s=0}^t \Big( {g_i(B_i^{k,s})} -{\frac{r}{m}}(P^k)^T\nabla f_i(x^k)+\frac{1}{\eta\tau}(Y_i^k-\overline{Y}^k)\Big)\|^2\\\notag
&\le 2 \| \eta \sum_{s=0}^t \left({g_i(B_i^{k,s})} {-{\frac{r}{m}}(P^k)^T\nabla f_i(x^k)}\right) \|^2\\\notag
& + 2\|\frac{t+1}{\tau}Y_i^k -\frac{t+1}{\tau}\overline{Y}^k \|^2.\notag
\end{align}    
Here, the notation $\pm a$ indicates the operation of adding and then subtracting $a$. In the second equality, we use 
\[
\overline{Y}^k = \eta \tau \frac{r}{m} (P^k)^T \nabla f(x^k),
\] 
and in the third equality, we substitute $c_i^k$ as defined in \eqref{eq-def-c}.

%
Substituting \eqref{eq-33} into \eqref{eq-32}, we get
\begin{equation}\label{eq-34}
\begin{aligned}
&\|B_i^{k,t+1}\|^2
\le 4 \underbrace{\| \eta \sum_{s=0}^t \left({g_i(B_i^{k,s})} {-{\frac{r}{m}}(P^k)^T\nabla f_i(x^k)} \right)\|^2}_{\rm(II)}\\
&+ 4\|\frac{t+1}{\tau}(Y_i^k -\overline{Y}^k) \|^2 +2\|{\frac{r}{m}}\eta(t+1)(P^k)^T\nabla f(x^k) \|^2.
\end{aligned}    
\end{equation}
To establish a recursion for $\|B_i^{k,t+1}\|^2$, we apply the $L_f$-smoothness property to bound term (II), which yields
\begin{align}\label{eq-b33}
&{\rm(II)}
\le \eta^2 {\theta_r^2} \| (P^k)^T\sum_{s=0}^t \left(\hat{\nabla} f_i(x^k+P^k B_i^{k,s}) \right.\\\notag
& \left.\pm \nabla f_i(x^k+P^k B_i^{k,s}) - {\nabla f_i(x^k)}\right) \|^2\\\notag
&\le \eta^2{\theta_r^2} \theta_m 2\Big( \! \underbrace{\| \sum_{s=0}^t\! \left( \hat{\nabla} f_i(x^k\!+\!P^k B_i^{k,s})\! -\! \nabla f_i(x^k\!+\!P^k B_i^{k,s}) \right) \!\|^2\!}_{\rm (II.1)}\\\notag
&+\|\sum_{s=0}^t \left(\nabla f_i(x^k+P^k B_i^{k,s}) -\nabla f_i(x^k) \right)\|^2 \Big), 
\end{align}
where we use $\theta_r=\max_l \frac{r_l}{m_l}$ in the first inequality.

For term {\rm (II.1)}, by taking the total expectation and following the derivation in [5, Lemma~4], we obtain
\begin{align}\label{eq-sgd-var}
&\bE{\rm (II.1)}=\\ \notag
&\bE \sum_{s_1=0}^t \sum_{s_2=0}^t \left\langle \hat{\nabla} f_i(x^k\!+\!P^k B_i^{k,s_1})\! -\! \nabla f_i(x^k\!+\!P^k B_i^{k,s_1}), \right.\\\notag
& \hspace{1.5cm} \left.\hat{\nabla} f_i(x^k\!+\!P^k B_i^{k,s_2}) - \nabla f_i(x^k\!+\!P^k B_i^{k,s_2}) \right\rangle
\\\notag
&= \bE \sum_{s=0}^t  \| \hat{\nabla} f_i(x^k+P^k B_i^{k,s}) - \nabla f_i(x^k+P^k B_i^{k,s})\|^2 \\\notag
&\le (t+1){\sigma^2}/{b},
\end{align}    
where we use Assumption \ref{asm-sgd-var} in the last inequality. 
Substituting \eqref{eq-sgd-var} into \eqref{eq-f-23} yields
\begin{align}\label{eq-f-23}
&\bE{\rm (II)}\\  \notag
&\le \eta^2 {\theta_r^2}\theta_m 2 (t\!+\!1)\frac{\sigma^2}{b} \!+\! \eta^2 {\theta_r^2}\theta_m 2(t\!+\!1) L_f^2 \sum_{s=0}^{t} \bE \| P^k B_i^{k,s} \|^2\\\notag
&\le {\eta^2 {\theta_r^2}\theta_m 2 (t+1)\frac{\sigma^2}{b}}+ \eta^2 {\theta_r^2}\theta_m^3 2(t+1)  L_f^2  \sum_{s=0}^{t} \bE\| B_i^{k,s} \|^2.   
\end{align}    
Thus, combining \eqref{eq-f-23} with \eqref{eq-34}, we obtain
\begin{align}\label{eq-36}
&\bE\|B_i^{k,t+1}\|^2 \\ \notag
&\le  8\eta^2 {\theta_r^2}\theta_m^3 (t+1)  L_f^2 \sum_{s=0}^{t} \bE\| B_i^{k,s} \|^2  \\\notag
&+\! \underbrace{{8\eta^2{\theta_r^2}  \theta_m\tau\frac{\sigma^2}{b}}\!+\!4\bE\|Y_i^k\! -\!\overline{Y}^k \|^2 +2\eta^2{\theta_r^2} \tau^2\bE\|(P^k)^T\nabla f(x^k) \|^2}_{:=A^k}.\notag
\end{align}    
Let us define $S_i^{k,t}:=\sum_{s=0}^t \bE\|B_i^{k,s}\|^2 $, then we have 
\begin{equation}\label{eq-b43}
\begin{aligned}
\bE\|B_i^{k,t+1}\|^2 = S_i^{k,t+1}-S_i^{k,t}.
\end{aligned}    
\end{equation}
Substituting \eqref{eq-b43} to \eqref{eq-36}, we have
\begin{equation}\label{eq-b36}
\begin{aligned}
S_i^{k,t+1}-S_i^{k,t}\le 8\eta^2 {\theta_r^2}\theta_m^3 (t+1)  L_f^2 S_i^{k,t} + A^k. 
\end{aligned}    
\end{equation}
Let us choose a small $\eta$ such that
\begin{equation}\label{eq-stepsize-1}
8\eta^2 {\theta_r^2}\theta_m^3 (t+1)  L_f^2 \le \frac{1}{8\tau },   
\end{equation}
then \eqref{eq-b36} yields 
\begin{equation}\label{eq-41}
\begin{aligned}
S_i^{k,t+1} &\le \left(1+{1}/{(8\tau) }\right) S_i^{k,t} + A^k. 
\end{aligned}    
\end{equation}
By repeatedly applying \eqref{eq-41}, we conclude that
\begin{equation}\label{drift-err-i}
\begin{aligned}
S^{k,\tau-1}_{i}
\leq A^k \sum_{t=0}^{\tau-2}\left(1+{1}/{(8\tau)}\right)^{t}
\le 1.15 \tau A^k,
\end{aligned}
\end{equation}
where we use $\sum_{l=0}^{\tau-2}\left(1+1/{(8\tau)}\right)^l \leq\sum_{l=0}^{\tau-2} \exp \left({ l}/{(8\tau)}\right) \leq\sum_{l=0}^{\tau-2} \exp (1/8) \le 1.15\tau$.
Substituting $A^k$ into \eqref{drift-err-i} yields
\begin{equation*}
\begin{aligned}
&\sum_{t=0}^{\tau-1} \bE\|B_i^{k,t}\|^2 \\
&\le 5 \tau \bE \| Y_i^k -\overline{Y}^k \|^2 +5 \eta^2{\theta_r^2} \tau^3 \bE\|(P^k)^T\nabla f(x^k) \|^2 \\&+10\eta^2{{\theta_r^2} \theta_m\tau^2 \frac{\sigma^2}{b}},   
\end{aligned}    
\end{equation*} 
which gives Lemma \ref{lem-B}. 

\subsection{ Proof of Lemma \ref{lem-Y}}\label{supp-lem2}

By the definition of $Y_i^k$ in \eqref{eq-def-Y}, we have
\begin{equation*}
\begin{aligned}
&\frac{1}{n}\sum_{i=1}^n \bE\| Y_i^{k+1} - \overline{Y}^{k+1} \|^2\\
&= \eta^2\tau^2\frac{1}{n}\sum_{i=1}^n \bE\|(P^{k+1})^T\Big( {\frac{r}{m}}{\nabla f_i(x^{k+1})} \\
&-P^k \frac{1}{\tau}\sum_{t=0}^{\tau-1} g_i(B_i^{k,t})  + P^k\frac{1}{\tau}\sum_{t=0}^{\tau-1} \bar{g}(B_i^{k,t})-  {\frac{r}{m}} {\nabla f(x^{k+1})}  \Big) \|^2\\
&\le \! \frac{\eta^2\tau^2}{n} \!\sum_{i=1}^n \bE\|(P^{k+1})^T \Big(\! {\frac{r}{m}}{\nabla f_i(x^{k+1})} \! -\! P^k\frac{1}{\tau}\sum_{t=0}^{\tau-1} g_i(B_i^{k,t}) \!\Big)\|^2\\
&\le \eta^2\tau^2\frac{1}{n}\sum_{i=1}^n \bE\|  {\frac{r}{m}}{\nabla f_i(x^{k+1})}  - P^k\frac{1}{\tau}\sum_{t=0}^{\tau-1} g_i(B_i^{k,t}) \|^2,
\end{aligned}    
\end{equation*}
where we use $\frac{1}{n}\sum_{i=1}^n\|Y_i-\overline{Y}\|^2\le \frac{1}{n}\sum_{i=1}^n\|Y_i\|^2$ in the first inequality, use the fact that $P^{k+1}$ is independent of the term $ {\frac{r}{m}}{\nabla f_i(x^{k+1})}  - P^k\frac{1}{\tau}\sum_{t=0}^{\tau-1} g_i(B_i^{k,t})$ in the last inequality. 
Adding and subtracting $\nabla f_i(x^k)$, the above inequality yields  
\begin{align}\label{eq-44}
&\frac{1}{n}\sum_{i=1}^n\bE\| Y_i^{k+1} - \overline{Y}^{k+1} \|^2\\\notag
&\le\!  \frac{\eta^2\tau^2}{n}\sum_{i=1}^n \bE\| {\frac{r}{m}}\! {\nabla f_i(x^{k+1})}\!\pm\! \frac{r}{m}\!\nabla f_i(x^k) \!-\! P^k\frac{1}{\tau}\sum_{t=0}^{\tau-1} \!g_i(B_i^{k,t}) \|^2 \\\notag
&\le \!2 \eta^2 {\theta_r^2}\tau^2  L_f^2  \bE\|x^{k+1}-x^k \|^2 \\\notag
&+\! \frac{2 \eta^2 {\theta_r^2} \tau^2}{n}\!\sum_{i=1}^n \bE\Big\| \nabla\! f_i(x^k) \! \pm\! P^k(P^k)^T\frac{1}{\tau}\sum_{t=0}^{\tau-1} \nabla f_i(x^k\!+\!P^k B_i^{k,t})  \\\notag
&-P^k(P^k)^T\frac{1}{\tau}\sum_{t=0}^{\tau-1}\hat{\nabla} f_i(x^k+P^k B_i^{k,t})\Big\|^2\\\notag
&\le\! 2 \eta^2 {\theta_r^2}\tau^2 L_f^2  \bE\|x^{k+1}-x^k \|^2 \\\notag
&+ \! \frac{4 \eta^2{\theta_r^2}\tau^2}{n} \!\sum_{i=1}^n \!\Big(\!\underbrace{\!\bE\| \nabla\! f_i(x^k) \! -\! P^k(\!P^k\!)^T\!\frac{1}{\tau}\!\sum_{t=0}^{\tau-1} \nabla \!f_i(x^k\!+\!P^k B_i^{k,t}) \|^2\!}_{\rm(III.1)}\\\notag
&+ \!\underbrace{\!\bE\| P^k(P^k)^T\!\frac{1}{\tau}\!\sum_{t=0}^{\tau-1} (\nabla\! f_i(x^k\!+\!P^k B_i^{k,t}) \!-\!\hat{\nabla}\! f_i(x^k\!+\!P^k B_i^{k,t}))\|^2\!}_{\rm(III.2)}\!\Big).   
\end{align}

To bound the term {\rm (III.1)}, we first establish an upper bound for
$ \bE \|(I-P^k(P^k)^T)\nabla f_i(x^k)\|^2. $ 
Let $\mathcal F_k$ represent the algorithmic history immediately before $P^k$ is sampled. Conditional on this history, $x^k$ and $\nabla f_i(x^k)$ are fixed, while $P^k$ remains an independent random matrix. Therefore,
\begin{equation*}
\begin{aligned}
&
\bE
\|
(I-P^k(P^k)^T)
\nabla f_i(x^k)
\|^2
\\
&=
\sum_{l=1}^L
\bE
\|
(I-P_l^k(P_l^k)^T)
\nabla_{x_l}f_i(x^k)
\|_F^2
\\
&=
\sum_{l=1}^L
\bE\!\left[
\operatorname{Tr}
\!\left(
\nabla_{x_l}f_i(x^k)^T
(I-P_l^k(P_l^k)^T)^2
\nabla_{x_l}f_i(x^k)
\right)
\right]
\\
&=
\sum_{l=1}^L
\bE\!\left[
\bE
\!\left[
\operatorname{Tr}
\!\left(
\nabla_{x_l}f_i(x^k)^T
(I-P_l^k(P_l^k)^T)^2
\nabla_{x_l}f_i(x^k)
\right)
\middle|
\mathcal F_k
\right]
\right]
\\
&=
\sum_{l=1}^L
\bE\!\left[
\operatorname{Tr}
\!\left(
\nabla_{x_l}f_i(x^k)^T
\bE
\!\left[
(I-P_l^k(P_l^k)^T)^2
\middle|
\mathcal F_k
\right]
\nabla_{x_l}f_i(x^k)
\right)
\right].
\end{aligned}
\end{equation*}
Since $P_l^k$ is independent of $\mathcal F_k$,
\begin{equation*}
\begin{aligned}
&
\bE
\!\left[
(I-P_l^k(P_l^k)^T)^2
\middle|
\mathcal F_k
\right]
\\
&=
\bE
\!\left[
I
-2P_l^k(P_l^k)^T
+
P_l^k(P_l^k)^TP_l^k(P_l^k)^T
\right]
\\
&=
I
-2\bE[P_l^k(P_l^k)^T]
+
\bE
\!\left[
P_l^k
\left((P_l^k)^TP_l^k\right)
(P_l^k)^T
\right]
\\
&=
I
-2I
+
\frac{m_l}{r_l}
\bE
[P_l^k(P_l^k)^T]
\\
&=
\left(
\frac{m_l}{r_l}-1
\right)
I,
\end{aligned}
\end{equation*}
where we use
$\bE[P_l^k(P_l^k)^T]=I$
and
$(P_l^k)^TP_l^k=\frac{m_l}{r_l}I$.
Substituting the above identity yields
\begin{align}\label{review-eq-1}
&\bE
\|
(I-P^k(P^k)^T)
\nabla f_i(x^k)
\|^2\\ \notag 
=&
\sum_{l=1}^L
\left(
\frac{m_l}{r_l}-1
\right)
\bE
\|
\nabla_{x_l}f_i(x^k)
\|_F^2
\\\notag 
\le&
(\theta_m-1)
\bE
\|
\nabla f_i(x^k)
\|^2.
\end{align}

Consequently, in \eqref{review-eq-1}, the projection-induced error for each client is bounded by its expected squared gradient norm. Taking expectations in Assumption~\ref{asm-Bg} gives the bound required in the subsequent analysis:
$ \frac1n
\sum_{i=1}^n
\bE
\|
\nabla f_i(x^k)
\|^2 \le B_g^2. $
This assumption is required solely due to the projection operator. In the full-space case ($P^k = \mathrm{I}$), $(I-P^k(P^k)^T)$ vanishes identically, and hence Assumption~\ref{asm-Bg} is no longer needed.
For the term (III.1), we have 
\begin{align}\label{eq-b44}
&{\rm(III.1)} = \bE\|\nabla f_i(x^k)\pm  P^k (P^k)^T \nabla f_i(x^k) \\ \notag
&-P^k(P^k)^T \frac{1}{\tau}\sum_{t=0}^{\tau-1} {\nabla f_i(x^k+P^k B_i^{k,t})} \|^2 \\\notag
&\le 2 \bE \|\left(I\!-\!P^k(P^k)^T\right) \nabla f_i(x^k)\|^2 + 
 2 \theta_m^2 L_f^2 \frac{1}{\tau } \sum_{t=0}^{\tau-1}  \bE\|P^k B_i^{k,t}\|^2 \\\notag
 &\le2 \left(\theta_m-1 \right)\bE\|\nabla f_i(x^k)\|^2 + 
 2 \theta_m^2 L_f^2 \frac{1}{\tau } \sum_{t=0}^{\tau-1}  \bE\|P^k B_i^{k,t}\|^2.\notag
\end{align}

For the term (III.2), we have 
$$
{\rm (III.2)}\le \theta_m^2 \frac{1}{\tau}\frac{\sigma^2}{b},
$$
where we use similar derivations as in \eqref{eq-sgd-var}. 
Substituting (III.1) and (III.2) into \eqref{eq-44} gives that
\begin{align}
&\frac{1}{n}\sum_{i=1}^n \bE\| Y_i^{k+1} - \overline{Y}^{k+1} \|^2\\\notag
&\le 2 \eta^2 {\theta_r^2}\tau^2  L_f^2  \bE\|x^{k+1}-x^k \|^2 + {4 \eta^2{\theta_r^2}\tau^2  \theta_m^2 \frac{1}{\tau}\frac{\sigma^2}{b} }\\\notag
&+4\eta^2{\theta_r^2}\tau^2  2 {\theta_m^3} L_f^2  \frac{1}{n\tau}\left(5 \tau  \sum_{i=1}^n \bE\| Y_i^k -\overline{Y}^k \|^2\right. \\\notag
& \left.+5 \eta^2 {\theta_r^2} \tau^3 n \bE\|(P^k)^T\nabla f(x^k) \|^2+10\eta^2{{\theta_r^2} \theta_m\tau^2 n\frac{\sigma^2}{b}} \right) \\\notag
&+  8 \eta^2{\theta_r^2}\tau^2  (\theta_m-1)\frac{1}{n}\sum_{i=1}^n\bE\|\nabla f_i(x^k)\|^2.   
\end{align}

If the step size $\eta$ satisfies
\begin{equation}\label{eq-stepsize-0}
4\eta^2{\theta_r^2}\tau^2  2 {\theta_m^3} L_f^2  \frac{1}{n\tau} 10\eta^2{{\theta_r^2} \theta_m\tau^2 n\frac{\sigma^2}{b}}  \le { \eta^2{\theta_r^2} \tau^2  \theta_m^2 \frac{1}{\tau}\frac{\sigma^2}{b} }  
\end{equation}
and after substituting $ \frac{1}{n}\sum_{i=1}^n \mathbb{E}\|\nabla f_i(x^k)\|^2\le B_g^2$, we complete the proof of Lemma \ref{lem-Y}.

\end{document}